\documentclass{article}

\usepackage{times}
\usepackage{graphicx} 
\usepackage{subfigure}

\usepackage{natbib}
\usepackage{multirow}
\usepackage{color}
\usepackage{amssymb,amsfonts}
\usepackage{amsmath,amssymb}
\usepackage{float}
\usepackage{comment}
\usepackage{verbatim}
\usepackage{fullpage}
\usepackage{appendix,stfloats}
\usepackage{bbm}
\usepackage{blkarray}
\usepackage{verbatim}
\usepackage{algorithm}
\usepackage[noend]{algpseudocode}
\makeatletter
\def\BState{\State\hskip-\ALG@thistlm}
\makeatother

\newcommand{\argmin}{\operatornamewithlimits{argmin}}

\usepackage{hyperref}

\newtheorem{theorem}{Theorem}[section]
\newtheorem{lemma}[theorem]{Lemma}

\newtheorem{corollary}[theorem]{Corollary}

\newenvironment{proof}[1][Proof]{\begin{trivlist}
\item[\hskip \labelsep {\bfseries #1}]}{\end{trivlist}}
\newenvironment{definition}[1][Definition]{\begin{trivlist}
\item[\hskip \labelsep {\bfseries #1}]}{\end{trivlist}}

\newcommand\Algphase[1]{%
\vspace*{-.7\baselineskip}\Statex\hspace*{\dimexpr-\algorithmicindent-2pt\relax}\rule{0.5\textwidth}{0.4pt}%
\Statex\hspace*{-\algorithmicindent}\textbf{#1}%
\vspace*{-.7\baselineskip}\Statex\hspace*{\dimexpr-\algorithmicindent-2pt\relax}\rule{0.5\textwidth}{0.4pt}%
}

\usepackage[accepted]{icml2015}

\icmltitlerunning{Feature-Budgeted Random Forest}

\begin{document}

\twocolumn[
\icmltitle{Feature-Budgeted Random Forest}

\icmlauthor{Feng Nan}{fnan@bu.edu}
\icmladdress{Boston University}
\icmlauthor{Joseph Wang}{joewang@bu.edu}
\icmladdress{Boston University}
\icmlauthor{Venkatesh Saligrama}{srv@bu.edu}
\icmladdress{Boston University}

\icmlkeywords{feature-budgeted learning, decision trees}

\vskip 0.3in
]

\begin{abstract}
We seek decision rules for {\it prediction-time cost reduction}, where complete data is available for training, but during prediction-time, each feature can only be acquired for an additional cost. We propose a novel random forest algorithm to minimize prediction error for a user-specified {\it average} feature acquisition budget. While random forests yield strong generalization performance, they do not explicitly account for feature costs and furthermore require low correlation among trees, which amplifies costs. Our random forest grows trees with low acquisition cost and high strength based on greedy minimax cost-weighted-impurity splits. Theoretically, we establish near-optimal acquisition cost guarantees for our algorithm. Empirically, on a number of benchmark datasets we demonstrate superior accuracy-cost curves against state-of-the-art prediction-time algorithms.
\end{abstract}

\section{Introduction}
In many applications such as surveillance and retrieval, we acquire measurements for an entity, and features for a query in order to make a prediction. Features can be expensive and complementary, namely, knowledge of previously acquired feature values often renders acquisition of another feature redundant. In these cases, the goal is to maximize prediction performance given a constraint on the average feature acquisition cost. Our proposed approach is to learn decision rules for prediction-time cost reduction~\cite{kanani_melville} from training data in which the full set of features and ground truth labels are available for training.

We propose a novel random forest learning algorithm to minimize prediction error for a user-specified {\it average} feature acquisition budget.
Random forests~\cite{Breiman} construct a collection of trees, wherein each tree is grown by random independent data sampling \& feature splitting, producing a collection of independent identically distributed trees. The resulting classifiers are robust, are easy to train, and yield strong generalization performance.

Although well suited to unconstrained supervised learning problems, applying random forests in the case of prediction-time budget constraints presents a major challenge. First, random forests do not account for feature acquisition costs. If two features have similar utility in terms of power to classify examples but have vastly different costs, random forest is just as likely to select the high cost feature as the low cost alternative. This is obviously undesirable. Second, a key element of random forest performance is the diversity amongst trees \cite{Breiman}. Empirical evidence suggest a strong connection between diversity and performance, and generalization error is bounded not only with respect to the strength of individual trees but also the correlation between trees \cite{Breiman}. High diversity amongst trees constructed without regard for acquisition cost results in trees using a wide range of features, and therefore a high acquisition cost (See Section \ref{sec:experiments}).

Thus, ensuring a low acquisition cost on the forest hinges on growing each tree with high discriminative power and low acquisition cost. To this end, we propose to learn decision trees that incorporates feature acquisition cost. Our random forest grows trees based on greedy minimax cost-weighted-impurity splits. Although the problem of learning decision trees with optimally low-cost is computationally intractable, we show that our greedy approach outputs trees whose cost is closely bounded with respect to the optimal cost. Using these low cost trees, we construct random forests with high classification performance and low prediction-time feature acquisition cost.

Abstractly, our algorithm attempts to solve an empirical risk minimization problem subject to a budget constraint. At each step in the algorithm, we add low-cost trees to the random forest to reduce the empirical risk until the budget constraint is met. The resulting random forest adaptively acquires features during prediction time, with features only acquired when used by a split in the tree. 
In summary, our algorithm is greedy and easy to train. It can not only be parallelized, but also lends itself to distributed databases. Empirically, it does not overfit and has low generalization error. Theoretically, we can characterize the feature acquisition cost for each tree and for the random forest. Empirically, on a number of benchmark datasets we demonstrate superior accuracy-cost curves against state-of-the-art prediction-time algorithms.

\textbf{Related Work:} The problem of learning from full training data for prediction-time cost reduction~\cite{mackay,kanani_melville} has been extensively studied. One simple structure for incorporating costs into learning is through detection cascades \cite{viola01,zhang:2010,chen:2012}, where cheap features are used to discard examples belonging to the negative class. Different from our apporach these approaches require a fixed order of features to be acquired and do not generalize well to multi-class. Bayesian approaches have been proposed which model the system as a POMDP \cite{ji:2007,kapoor:2009,Gao+Koller:NIPS11}, however they require estimation of the underlying probability distributions. To overcome the need to estimate distributions, reinforcement learning \cite{karayev13,busa2012fast,dulac2011datum} and imitation learning \cite{he2012imitation} approaches have also been studied, where the reward or oracle action is predicted, however these generally require classifiers capable of operating on a wide range of missing feature patterns.

Supervised learning approaches with prediction-time budgets have previously been studied under an empirical risk minimization framework to learn budgeted decision trees \cite{xu2013cost,ASTC_AAAI14,trapeznikov:2013b,wang2014lp,wang2014model}. In this setting, construction of budgeted decision cascades or trees has been proposed by learning complex decision functions at each node and leaf, outputting a tree of classifiers which adaptively select sensors/features to be acquired for each new example. Common to these systems is a decision structure, which is a priori fixed. The entire structure is parameterized by complex decision functions for each node, which are then optimized using various objective functions. In contrast we build a random forest of trees where each tree is grown greedily so that global collection of random trees meets the budget constraint.

Construction of simple decision trees with low costs has also been studied for discrete function evaluation problems \cite{DiagnosisDeterminationSimultaneous,MoshkovGreedyAlgorithmwithWeightsforDecisionTreeConstruction,GroupBasedActiveLearning}. Different from our work these trees operate on discrete data to minimize function evaluations, with no notion of test time prediction or cost.

As for Random forests despite their widespread use in supervised learning, to our knowledge they have not been applied to prediction-time cost reduction.

\section{Feature-Budgeted Random Forest}
We first present the general problem of learning under prediction-time budgets similar to the formulation in \cite{trapeznikov:2013b,wang2014lp}. Suppose example/label pairs $(x,y)$ are distributed as $(x,y)\stackrel{d}{\sim} H$. The goal is to learn a classifier $f$ from a family of functions $\mathcal{F}$ that minimizes expected loss subject to a budget constraint:
\begin{align}\label{eq:budgetedLearning}
\min_{f \in \mathcal{F}} E_{xy}\left[L(y,f(x))\right],\,\, \text{s.t. } E_x\left[C\left(f,x\right)\right]\leq B,
\end{align}
where $L(y,\hat{y})$ is a loss function, $C(f,x)$ is the cost of evaluating the function of $f$ on example $x$ and $B$ is a user specified budget constraint. In this paper, we assume that the feature acquisition cost $C(f,x)$ is a modular function of the support of the features used by function $f$ on example $x$, that is acquiring each feature has a fixed constant cost. Without the cost constraint, the problem is equivalent to a supervised learning problem, however, adding the cost constraint makes this a combinatorial problem \cite{xu2013cost}.
In practice, we are not given the distribution but instead are given a set of training data $(x_1,y_1),\ldots,(x_n,y_n)$ drawn IID with $(x_i,y_i)\stackrel{d}{\sim} H$. We can then minimize 
the empirical loss subject to a budget constraint:
\begin{align}\label{eq:budgetedLearning_erm}
\min_{f \in \mathcal{F}}\frac{1}{n}\sum_{i=1}^{n}L(y_i,f(x_i)),\,\, \text{s.t. } \frac{1}{n}\sum_{i=1}^{n}C\left(f,x_i\right)\leq B.
\end{align}
In our context the classifier $f$ is a random forest, ${\cal T}$, consisting of $K$ random trees, $D_1,\, D_2,\,\ldots,D_K$, that are learnt on training data. Consequently, the expected cost for an instance $x$ during prediction-time can be written as follows:
\begin{align} \label{eq:costbound}
E_f \left[E_x\left[C\left(f,x\right)\right]\right ] \leq  \sum_{j=1}^K E_{D_j}\left[ E_x\left[C\left(D_j,x\right)\right]\right ] 
\end{align}
where, in the RHS we are averaging with respect to the random trees. As the trees in a random forest are identically distributed the RHS scales with the number of trees.
This upper-bound captures the typical behavior of a random forest due to the low feature correlation among trees.

As a result of this observation, the problem of learning a budgeted random forest can be viewed as equivalent to the problem of finding decision trees with low expected evaluation cost and error. This motivates our algorithm \textsc{BudgetRF}, where greedily constructed decision trees with provably low feature acquisition cost are added until the budget constraint is met according to validation data. The returned random forest is  a feasible solution to \eqref{eq:budgetedLearning} with strong empirical performance.



\subsection{Our Algorithm}
\textbf{During Training:} As shown in Algorithm \ref{algo:BudgetRF}, there are seven inputs to \textsc{BudgetRF}: impurity function $F$, prediction-time feature acquisition budget $B$, a cost vector $C\in \Re^m$ that contains the acquisition cost of each feature, training class labels $y_{tr}$ and data matrix $X_{tr} \in \Re^{n\times m}$,  where $n$ is the number of samples and $m$ is the number of features, validation class labels $y_{tv}$ and data matrix $X_{tv}$. Note that the impurity function $F$ needs to be \emph{admissible}, which essentially means monotone and supermodular. We defer the formal definition and theoretical results to Section \ref{sec:greedyBound}. For now it is helpful to think of an impurity function $F$ as measuring the heterogeneity of a set of examples. Intuitively, $F$ is large for a set of examples with mostly different labels and small for a set with mostly the same label.

\begin{algorithm}
\caption{{\textbf{\textsc{BudgetRF}}}}\label{algo:BudgetRF}
\begin{algorithmic}[1]
\Procedure{BudgetRF($F,B,C,ytr,Xtr,ytv,Xtv$)}{}
\State $\mathcal{T} \gets \emptyset$.
\While{Average cost using validation set on $\mathcal{T}$ $\leq B$}
\State Randomly sample $n$ training data with replacement to form $X^{(i)}$ and $y^{(i)}$.
\State Train $T \gets$ {\textsc{GreedyTree}}($F,C,y^{(i)},X^{(i)}$).
\State $\mathcal{T} \gets \mathcal{T} \cup T$.
\EndWhile
\State \Return $\mathcal{T}\backslash T$.
\EndProcedure
\Algphase{Subroutine - \textsc{GreedyTree}}
\Procedure{GreedyTree($F,C,y,X$)}{}
\State $S\gets (y,X)$ \Comment{the current set of examples}
\If {$F(S)=0$} \Return
\EndIf
\For {each feature $t = 1$ to $m$}
\State Compute
$R(t):= \underset{g_t\in \mathcal{G}_t}{\min}\underset{i\in \text{outcomes}}{\max} \frac{c(t)}{F(S)-F(S^i_{g_t})}$, \Comment{risk for feature $t$} \label{eq:risk}
\State where $S^i_{g_t}$ is the set of examples in $S$ that has outcome $i$ using classifier $g_t$ with feature $t$.
\EndFor
\State $\hat{t} \gets \argmin_t R(t)$
\State $\hat{g} \gets \underset{g_{\hat{t}} \in \mathcal{G}_{\hat{t}}}{\argmin} \underset{i\in \text{outcomes}}{\max} \frac{c(\hat{t})}{F(S)-F(S^i_{g_{\hat{t}}})}$
\State Make a node using feature $\hat{t}$ and classifier $\hat{g}$.
\For {each outcome $i$ of $\hat{g}$}
\State $\textsc{GreedyTree}(F,C,y^i_{\hat{g}},X^i_{\hat{g}})$ to append as child nodes.
\EndFor
\EndProcedure
\end{algorithmic}
\end{algorithm}

\textsc{BudgetRF} iteratively builds decision trees by calling \textsc{GreedyTree} as a subroutine on a sampled subset of examples from the training data until the budget $B$ is exceeded as evaluated using the validation data. The ensemble of trees are then returned as output.
 As shown in subroutine \textsc{GreedyTree}, the tree building process is greedy and recursive. If the given set of examples have zero impurity as measured by $F$, they are returned as a leaf node. Otherwise, compute the \emph{risk} $R(t)$ for each feature $t$, which involves searching for a classifier $g_t$ among the family of classifiers $\mathcal{G}_t$ that minimizes the maximum impurity among its outcomes. Intuitively, a feature with the least $R(t)$ can uniformly reduce the impurity among all its child nodes the most with the least cost. Therefore such a feature $\hat{t}$ is chosen along with the corresponding classifier $\hat{g}$. The set of examples are then partitioned using $\hat{g}$ to different child nodes at which \textsc{GreedyTree} is recursively applied. Note that we allow the algorithm to reuse the same feature for the same example in \textsc{GreedyTree}.

\textbf{During Prediction:} Given a test example and a decision forest $\mathcal{T}$ returned by \textsc{BudgetRF}, we run the example through each tree in $\mathcal{T}$ and obtained a predicted label from each tree. The final predicted label is simply the majority vote among all the trees.

Different from random forest, we incorporate feature acquisition costs in the tree building subroutine \textsc{GreedyTree} with the hope of reducing costs while maintaining low classification error. Our main theoretical contribution is to propose a broad class of \emph{admissible} impurity functions such that on any given set of $n'$ examples the tree constructed by \textsc{GreedyTree} will have max-cost bounded by $O(\log n')$ times the optimal max-cost tree.

%

\subsection{Bounding the Cost of Each Tree}\label{sec:greedyBound}
Given a set of examples $S$ with features and corresponding labels, a classification tree $D$ has a feature-classifier pair associated with each internal node. A test example is routed from the root of $D$ to a leaf node directed by the outcomes of the classifiers along the path; the test example is then labeled to be the majority class among training examples in the leaf node it reaches. The feature acquisition cost of an example $s\in S$ on $D$, denoted as $cost(D,s)$, is the sum of all feature costs incurred along the root-to-leaf path in $D$ traced by $s$. Note that if $s$ encounters a feature multiple times in the path, the feature cost contributes to $cost(D,s)$ only once because subsequent use of a feature already acquired for the test example incurs no additional cost.
We define the total max-cost as
$$Cost(D)=\underset{s\in S}{\max}  cost(D,s).$$ 
We aim to build a decision tree for any given set of examples such that the max-cost is minimized. Note that the max-cost criterion bounds the expected cost criterion of Eq.~\ref{eq:costbound}. While this bound could be loose we show later (see Sec.~\ref{sec:disc}) that by parameterizing a suitable class of impurity functions, the max-costs of our \textsc{GreedyTree} solution can be ``smoothened" so that it approaches the expected-cost.

First define the following terms: $n'$ is the number of examples input to {\textsc{GreedyTree}} and $m$ is the number of features, each of which has (a vector of) real values; $F$ is the given impurity function; $F(S)$ is the impurity on the set of examples $S$; $D_{F}$ is the family of decision trees with $F(L)=0$ for any of its leaf $L$;
each feature has a cost $c(t)$; a family of classifiers $\mathcal{G}_t$ is associated with feature $t$; $Cost_{F}(S)$ is the max-cost of the tree constructed by {\textsc{GreedyTree}} using impurity function $F$ on $S$;  and assume no feature is used more than once on the same example in the \emph{optimal} decision tree among $D_{F}$ that achieves the minimum max-cost, which we denote as $OPT(S)$ for the given input set of examples $S$. Note the assumption here is a natural one if the complexity of $\mathcal{G}_t$ is high enough.
We show the $O(\log n')$ approximation holds for the max-cost of the optimal testing strategy using the \textsc{GreedyTree} subroutine if the impurity function $F$ is admissible.
\begin{definition} 
A function $F$ of a set of examples is \emph{admissible} if it satisfies the following five properties: (1) Non-negativity: $F(G)\geq 0$ for any set of examples $G$; (2) Purity: $F(G)=0$ if $G$ consists of examples of the same class; (3) Monotonicity: $F(G)\geq  F(R), \forall R \subseteq G$;  (4) Supermodularty: $F(G\cup j)-F(G)\geq F(R\cup j) -F(R)$ for any $R\subseteq G$ and example $j\notin R$; (5) $\log(F(S))=O(\log n')$.
\end{definition}
Since the set $S$ is always finite, by scaling $F$ we can assume the smallest non-zero impurity of $F$ is 1.
Let $\tau$ and $\hat{g}_\tau$ be the first feature and classifier selected by {\textsc{GreedyTree}} at the root and let $S^i_{\hat{g}_\tau}$ be the set of examples in $S$ that has outcome $i$ using classifier $\hat{g}_\tau$.
Note the optimization of classifier in Line \eqref{eq:risk} of Algorithm \ref{algo:BudgetRF} needs not to be exact. We say {\textsc{GreedyTree}} is \emph{$\lambda$-greedy} if $\hat{g}_\tau$ is chosen such that
\begin{equation*}
\underset{i\in \text{outcomes}}{\max} \frac{c(\gamma)}{F(S)-F(S^i_{\hat{g}_\tau})} \leq \underset{g_t \in \mathcal{G}_{t}}{\min} \underset{i\in \text{outcomes}}{\max} \frac{\lambda c(t)}{F(S)-F(S^i_{g_{t}})},
\end{equation*}
for some constant $\lambda \geq 1$.
 By definition of max-cost,
\begin{equation*}
\frac{Cost_{F}(S)}{OPT(S)}\leq \frac{c(\tau)+\underset{i}{\max} Cost_{F}(S^i_{\hat{g}_\tau})}{OPT(S)},
\end{equation*}
because feature $\tau$ could be selected multiple times by \textsc{GreedyTree} along a path and the feature cost $c(\tau)$ contributes only once to the cost of the path.

Let $q$ be such that $Cost_F(S^q_{\hat{g}_\tau})=\underset{i}{\max} Cost_{F}(S^i_{\hat{g}_\tau})$. We first provide a lemma to lower bound the optimal cost, which will later be used to prove a bound on the cost of the tree.

\begin{lemma}\label{lemma:OPT_Wlowerbound}
Let $F$ be monotone and supermodular; let $\tau$ and $\hat{g}_\tau$ be the first feature and classifier chosen by {\textsc{GreedyTree}} $\lambda$-greedily on the set of examples $S$, then
\begin{equation*}
c(\tau)F(S)/(F(S)-F(S^q_{\hat{g}_\tau})) \leq \lambda OPT(S).
\end{equation*}
\end{lemma}
\begin{proof}
Let $D^*\in D_{F}$ be a tree with optimal max-cost. Let $v$ be an arbitrarily chosen internal node in $D^*$, let $\gamma$ be the feature associated with $v$ and $g^*_\gamma$ the corresponding classifier. Let $R\subseteq S$ be the set of examples associated with the leaves of the subtree rooted at $v$. Let $i$ be such that $c(\tau)/(F(S)-F(S^i_{\hat{g}_\tau}))$ is maximized. Let $g^{min}_\gamma=\underset{g_\gamma \in \mathcal{G}_{\gamma}}{\argmin} \underset{i\in \text{outcomes}}{\max} \frac{c(\gamma)}{F(S)-F(S^i_{g_{\gamma}})}$. Let $w$ be such that $c(\gamma)/(F(S)-F(S^w_{g^{min}_\gamma}))$ is maximized; similarly let $j$ be such that $c(\gamma)/(F(S)-F(S^j_{g^{*}_\gamma}))$ is maximized. We then have:
\begin{align}
&\frac{c(\tau)}{F(S)-F(S^q_{\hat{g}_\tau})} \leq \frac{c(\tau)}{F(S)-F(S^i_{\hat{g}_\tau})} \notag \leq
\frac{\lambda c(\gamma)}{F(S)-F(S^w_{g^{min}_\gamma})} \\
&\leq
\frac{\lambda c(\gamma)}{F(S)-F(S^j_{g^{*}_\gamma})} \leq \frac{\lambda c(\gamma)}{F(R)-F(R^j_{g^{*}_\gamma})}.\label{eq:lemma2}
\end{align}
The first inequality follows from the definition of $i$. The second inequality follows from the $\lambda$-greedy choice at the root. The third inequality follows from the minimization over classifiers given feature $\gamma$. To show the last inequality, we have to show $F(S)-F(S^j_{g^{*}_\gamma})\geq F(R)-F(R^j_{g^{*}_\gamma})$. This follows from the fact that $S^j_{g^{*}_\gamma} \cup R \subseteq S$ and $R^j_{g^{*}_\gamma} = S^j_{g^{*}_\gamma} \cap R$ and therefore $F(S)\geq F(S^j_{g^{*}_\gamma} \cup R) \geq F(S^j_{g^{*}_\gamma})+F(R)-F(R^j_{g^{*}_\gamma})$, where the first inequality follows from monotonicity and the second follows from the definition of supermodularity.

For a node $v$, let $S(v)$ be the set of examples associated with the leaves of the subtree rooted at $v$. Let $v_1,v_2,\dots,v_p$ be a root-to-leaf path on $D^*$ as follows: $v_1$ is the root of the tree, and for each $i=1,\dots,p-1$ the node $v_{i+1}$ is a child of $v_i$ associated with the branch of $j$ that maximizes $c(t_i)/(F(S)-F(S^j_{g^*_{t_i}}))$, where $t_i$ is the test associated with $v_i$. It follows from \eqref{eq:lemma2} that
\begin{equation}
\frac{[F(S(v_i))-F(S(v_{i+1}))]c(\tau)}{\lambda(F(S)-F(S^q_{\hat{g}_\tau}))}\leq c_{t_i}.
\end{equation}
Since the cost of the path from $v_1$ to $v_p$ is no larger than the max cost of the $D^*$, we have that
\begin{align*}
& OPT(S) \geq \sum_{i=1}^{p-1}c_{t_i} \\
 & \geq \frac{c(\tau)}{\lambda(F(S)-F(S^q_{\hat{g}_\tau}))}\sum_{i=1}^{p-1}(F(S(v_i))-F(S(v_{i+1}))\\
&=\frac{c(\tau)(F(S)-F(S(v_p))}{\lambda(F(S)-F(S^q_{\hat{g}_\tau}))}=\frac{c(\tau)F(S)}{\lambda(F(S)-F(S^q_{\hat{g}_\tau}))}.
\end{align*}
\end{proof}
The main theorem of this section is the following.
\begin{theorem} \label{thm:logn}
{\textsc{GreedyTree}} constructs a decision tree achieving $O(\log n')$-factor approximation of the optimal max-cost in $D_{F}$ on the set $S$ of $n'$ examples if  $F$ is admissible and no feature is used more than once on any path of the optimal tree.
\end{theorem}
\begin{proof}
This is an inductive proof:
\begin{align}
& \frac{Cost_{F}(S)}{OPT(S)} \leq \frac{c(\tau)+Cost_{F}(S^q_{\hat{g}_\tau})}{OPT(S)} \\
&\leq \frac{c(\tau)}{OPT(S)}+\frac{Cost_{F}(S^q_{\hat{g}_\tau})}{OPT(S^q_{\hat{g}_\tau})} \label{eq:thm1_1}
\\ & \leq \lambda \frac{F(S)-F(S^q_{\hat{g}_\tau})}{F(S)}+\frac{Cost_{F}(S^q_{\hat{g}_\tau})}{OPT(S^q_{\hat{g}_\tau})} \label{eq:thm1_2} \\
& \leq \lambda \log (\frac{F(S)}{F(S^q_{\hat{g}_\tau})})+\lambda \log (F(S^q_{\hat{g}_\tau}))+1 \label{eq:thm1_4}\\
& = \lambda \log (F(S))+1 =O(\log(n')). \label{eq:thm1_5}
\end{align}
The inequality in \eqref{eq:thm1_1} follows from the fact that $OPT(S) \geq OPT(S^q_{\hat{g}_\tau})$. \eqref{eq:thm1_2} follows from Lemma \ref{lemma:OPT_Wlowerbound}. The first term in \eqref{eq:thm1_4} follows from the inequality $\frac{x}{x+1} \leq \log(1+x)$ for $x>-1$ and the second term follows from the induction hypothesis that for each $G\subset S$, ${Cost_{F}(G)}/{OPT(G)}\leq \lambda \log(F(G))+1$. If $F(G)=0$ for some set of examples $G$, we define ${Cost_{F}(G)}/{OPT(G)}=1$.

We can verify the base case of the induction as follows.
if $F(G)=1$, which is the smallest non-zero impurity of $F$ on subsets of examples $S$, we claim that the optimal decision tree chooses the feature with the smallest cost among those that can reduce the impurity function $F$:
\begin{equation*}
OPT(G)=\min_{t|\exists g_t, \text{s.t. } F(G^i_{g_t})=0, \forall i\in \text{outcomes}} c(t).
\end{equation*}
Suppose otherwise, the optimal tree chooses first a feature $t$ with a child node $G'$ such that $F(G')=1$ and later chooses another feature $t'$ such that all the child nodes of $G'$ by $g_{t'}$ has zero impurity, then $t'$ could have been chosen in the first place to reduce all child nodes of $G$ to zero impurity by supermodularity of $F$.
On the other hand, $R(t)=\infty$ in {\textsc{GreedyTree}} for the features that cannot reduce impurity and $R(t)=c(t)$ for those features that can. So the algorithm would pick the feature among those that can reduce impurity and have the smallest cost. Thus, we have shown that ${Cost_{F}(G)}/{OPT(G)} = 1 \leq \lambda \log(F(G))+1$ for the base case.
\end{proof}

\subsection{Admissible Impurity Functions}
A wide range of functions falls into the class of admissible impurity functions. We employ a particular function called threshold-Pairs in our paper defined as
\begin{equation}\label{eq:hingedPairs}
F_\alpha(G)=\sum_{i\neq j}[[n^i_G-\alpha]_+[n^j_G-\alpha]_+-\alpha^2]_+,
\end{equation}
where $n_G^i$ denotes the number of objects in $G$ that belong to class $i$, $[x]_+=\max(x,0)$ and $\alpha$ is a threshold parameter.
We include the proof of the following lemma in the Appendix.
\begin{lemma}\label{lemma:F_admissible_multi}
$F_\alpha(G)$ is \emph{admissible}.
\end{lemma}

Neither entropy nor Gini index satisfies the notion of admissibility because they are not monotonic set functions, that is a subset of examples does not necessarily have a smaller entropy or Gini index compared to the entire set. Therefore traditional decision tree learning algorithms do not incorporate feature costs and have no guarantee on the max-cost as stated in our paper. We have studied more impurity functions that are admissible such as the polynomials and Powers family of functions. After conducting experiments on smaller datasets we noted that they do not offer significant advantage over the threshold-Pairs used in this paper. Please see Appendix for more details.

\subsection{Discussions of the Algorithm} \label{sec:disc}
Before concluding the \textsc{BudgetRF} algorithm and its analysis, we discuss further various design issues as well as their implications.

\textbf{Choice of threshold $\alpha$.}
In subroutine \textsc{GreedyTree}, each tree is greedily built until a minimum leaf impurity is met, then added to the random forest. The threshold $\alpha$ can be used to trade-off between average tree depth and number of trees. A lower $\alpha$ results in deeper trees with higher classification power and acquisition cost. As a result, fewer trees are added to the random forest before the budget constraint is met. Conversely, a higher $\alpha$ yields shallower trees with poorer classification performance, however due to the low cost of each tree, many are added to the random forest before the budget constraint is met. As such, $\alpha$ can be viewed as a bias-variance trade-off. In practice, it is selected using validation dataset.

Another observation we make is that the choice of $\alpha$ can potentially lead to different feature choice when used in \textsc{GreedyTree}.
To illustrate this point, consider the toy example in Figure \ref{fig:demo}. A set $G$ has 30 examples in class 1 (circles) and 30 examples in Class 2 (triangles). Two features $t_1$ and $t_2$ are available to the algorithm at equal cost. Feature $t_1$ has only one classifier in $\mathcal{G}_{t_1}$ as drawn on the upper left of the figure, which can separate 20 examples of Class 2 from the rest of the examples while $t_2$ has only one classifier in $\mathcal{G}_{t_2}$ as drawn on the lower left of the figure, which evenly divides the examples into halves with equal number of examples from Class 1 and Class 2 in either half. Intuitively, $t_2$ is not a useful feature from a classification point of view because it cannot separate examples based on class at all. This is reflected in the right plot of Figure \ref{fig:demo}: choosing $t_2$ increases cost but does not reduce classification error while choosing $t_1$ reduces the error to $\frac{1}{6}$. If $\alpha$ is set to 0 in the threshold-Pairs, feature $t_2$ will be chosen due to the fact that Pairs biases towards feature-classifiers with balanced outcomes. In contrast, setting $\alpha=8$ leads to feature $t_1$, and therefore may be preferable (see Appendix). 
\begin{figure}
\centering
\includegraphics[trim=2.5cm 5cm 3cm 0cm,angle=0,width=0.4\textwidth]{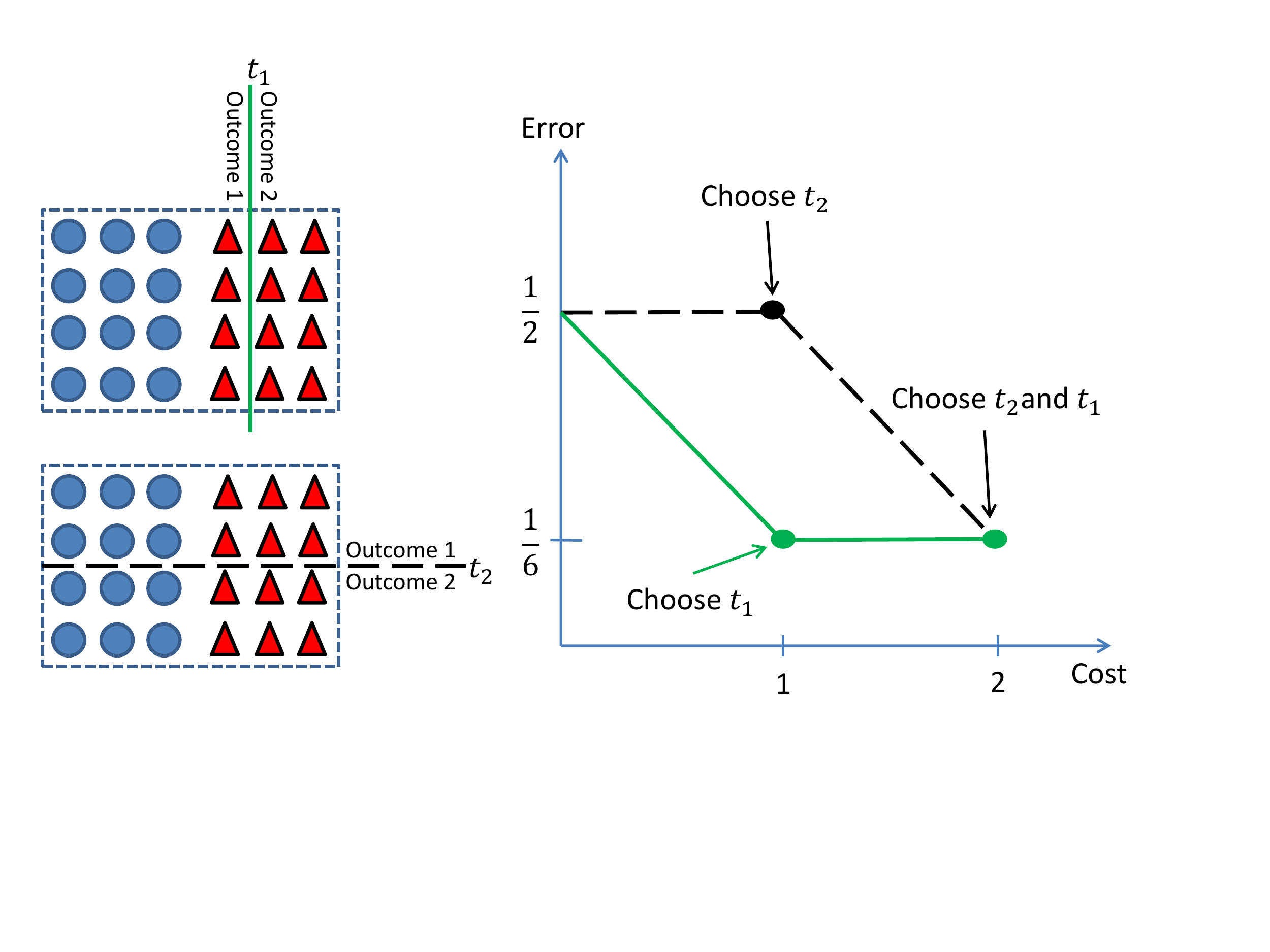}
\caption{Illustration of different $\alpha$ setting in threshold-Pairs for different greedy choice of features. The left two figures above show the feature-classifier outcomes of features $t_1$ and $t_2$. The right figure shows the classification error against cost (number of features). Here setting $\alpha=0$ leads to choosing $t_2$ because it prefers balanced splits; setting $\alpha=8$ leads to choosing $t_1$, which is better from an error-cost trade-off point of view.} \label{fig:demo}
\vspace{-0.1in}
\end{figure}

\textbf{Minimax-splits.}
The splitting criterion in the subroutine \textsc{GreedyTree} is based on the worst case impurity among child nodes, we call such splits minimax-splits as opposed to expected-splits, which is based on the expected impurity among child nodes. Using minimax-splits, our theoretical guarantee is a bound on the max-cost of individual trees. Note such minimax-splits have been shown to lead to expected-cost bound as well in the setting of GBS \cite{Nowak08generalizedbinary}; an interesting future research direction is to show whether minimax-splits can lead to a bound on the expected-cost of individual trees in our setting.

\textbf{Smoothened Max-Costs.}
We emphasize that by adjusting $\alpha$ in  threshold-Pairs function - essentially allowing some error, the max-costs of the \textsc{GreedyTree} solution can be ``smoothened" so that it approaches the expected-cost.
Consider the synthetic example as shown in Figure \ref{fig:synth}.
\begin{figure}
\centering
\includegraphics[trim=1.7cm 6cm 1.7cm 4cm,clip,width=0.35\textwidth]{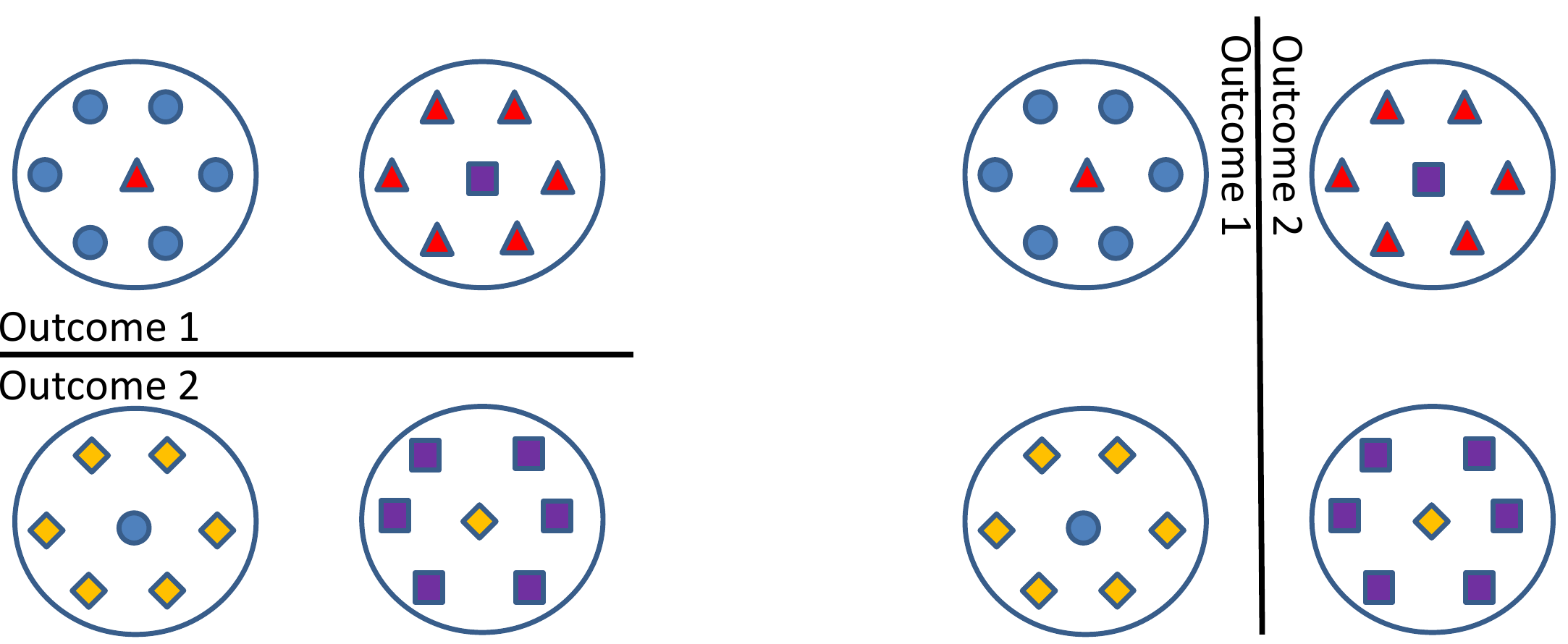}
\caption{A synthetic example to show max-cost of \textsc{GreedyTree} can be ``smoothened" to approach the expected-cost. The left and right figures above show the classifier outcomes of feature $t_1$ and $t_2$, respectively. } \label{fig:synth}
\vspace{-.5cm}
\end{figure}
\begin{figure}
\centering
\includegraphics[trim=5.75cm 4cm 6cm 4cm,clip,width=0.3\textwidth]{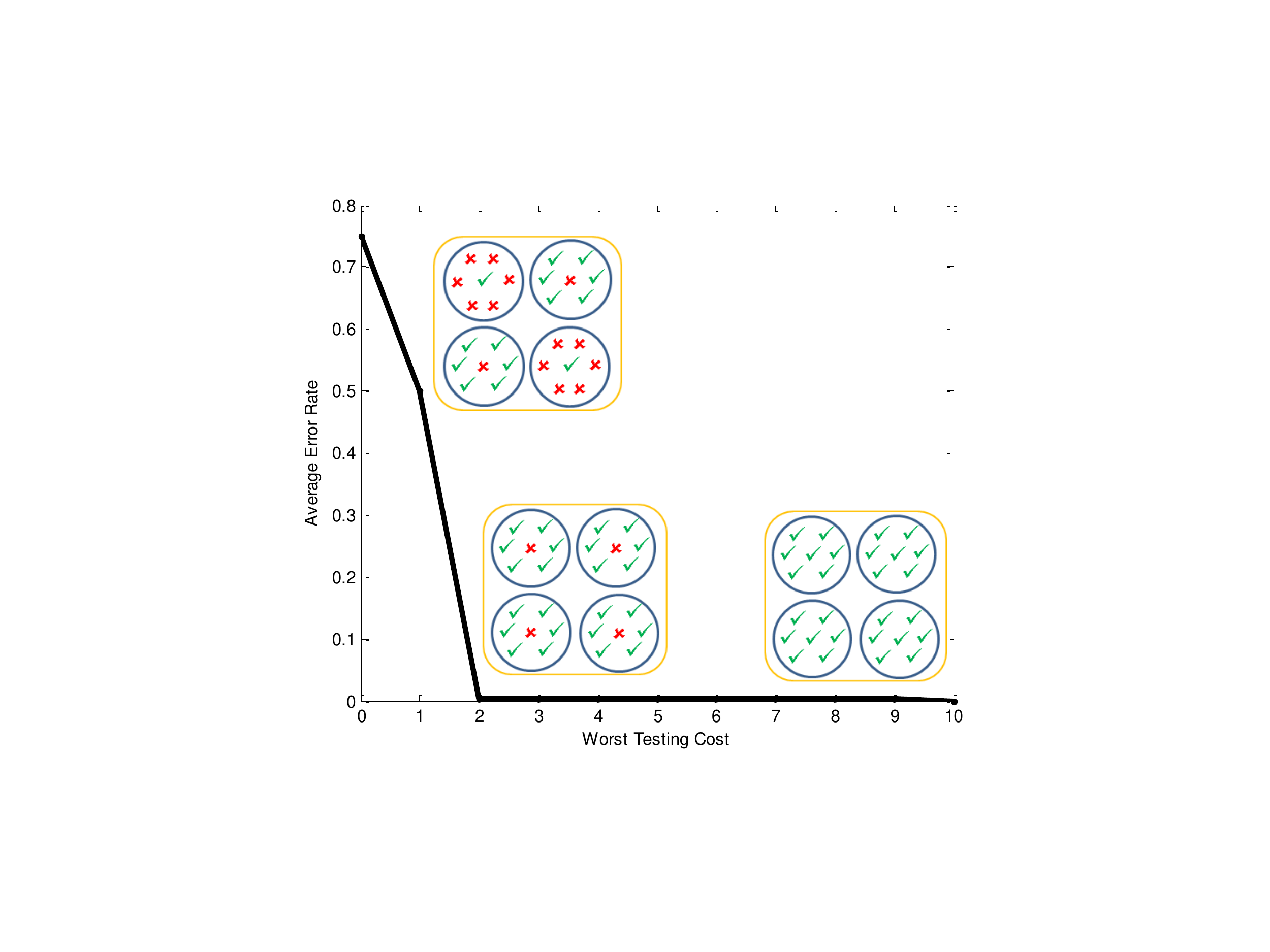}
\vspace{-.35cm}
\caption{The error-cost trade-off plot of the subroutine \textsc{GreedyTree} using threshold-Pairs on the synthetic example. $0.39\%$ error can be achieved using only a depth-2 tree but it takes a depth-10 tree to achieve zero error. } \label{fig:synthplot}
\vspace{-.5cm}
\end{figure}
Here we consider a multi-class classification example to demonstrate the effect of ``smoothened" max-cost of the tree approaching the expected-cost. Consider a data set composed of 1024 examples belonging to 4 classes with 10 binary features available. Assume that is no two examples that have the same set of feature values. Note that by fixing the acquisition order of the features, the set of feature values maps each example to an integer in the range $[0,1023]$. From this mapping, we give the examples in the ranges $[1,255]$ , $[257,511]$ , $[513,767]$, and  $[769,1023]$ the labels $1$, $2$, $3$, and $4$, respectively, and the examples $0$, $256$, $512$, and $768$ the labels $2$, $3$, $4$, and $1$, respectively (Figure \ref{fig:synth} shows the data projected to the first two features). Suppose each feature carries a unit cost. By Kraft's Inequality~\cite{cover}, the optimal max-cost in order to correctly classify every object is 10, however, using only $t_1$ and $t_2$ as selected by the greedy algorithm, leads to a correct classification of all but 4 objects, as shown in Figure \ref{fig:synthplot}. Thus, the max-cost of the early stopped tree is only 2 - much closer to the expected-cost.


\section{Experiments}\label{sec:experiments}
For establishing baseline comparisons we apply \textsc{BudgetRF} on 4 real world benchmarked datasets. The first one has varying feature acquisition costs in terms of computation time and the purpose is to show our algorithm can achieve high accuracy during prediction while saving massive amount of feature acquisition time. The other 3 datasets do not have explicit feature costs; instead, we assign a unit cost to each feature uniformly. The purpose is to demonstrate our algorithm can achieve low test error using only a small fraction of features. Note our algorithm is adaptive, meaning it acquires different features for different examples during testing. So the feature costs in the plots should be understood as an average of costs for all test examples.
We use CSTC \cite{xu2013cost} and ASTC \cite{ASTC_AAAI14} for comparison because they have been shown to have state-of-the-art cost-error performance. 
For comparison purposes we use the same configuration of training/validation/test splits as in ASTC/CSTC. The algorithm parameters for ASTC are set using the same configuration as in \cite{ASTC_AAAI14}. We report values for CSTC from \cite{ASTC_AAAI14}.
In all our experiments we use the threshold-Pairs \eqref{eq:hingedPairs} as impurity function. We use stumps as the family of classifiers $\mathcal{G}_t$ for all features t. The optimization of classifiers in line 12 of Algorithm \ref{algo:BudgetRF} is approximated by randomly generating 80, 40 and 20 stumps if the number of examples exceeds 2000, 500 and less than 500, respectively and select the best among them.  All results from our algorithm were obtained by taking an average of 10 runs and standard deviations are reported using error bars.
\begin{figure*}[htb!]
\vspace{-0.05in}
\centering
\subfigure[Yahoo! Rank]{\includegraphics[width=.39\linewidth,height=.29\linewidth]{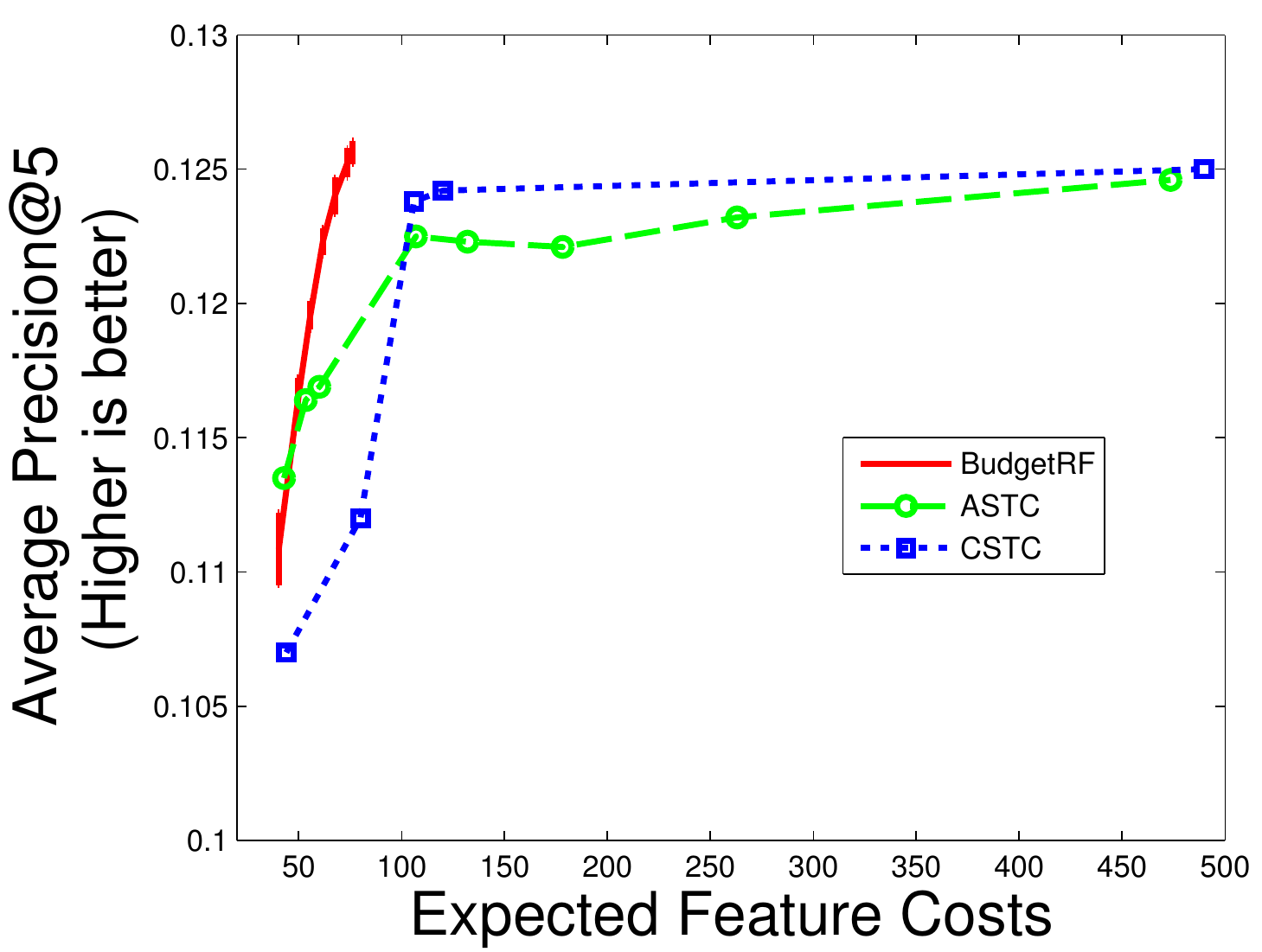}}
\subfigure[MiniBooNE]{\includegraphics[height=.29\linewidth]{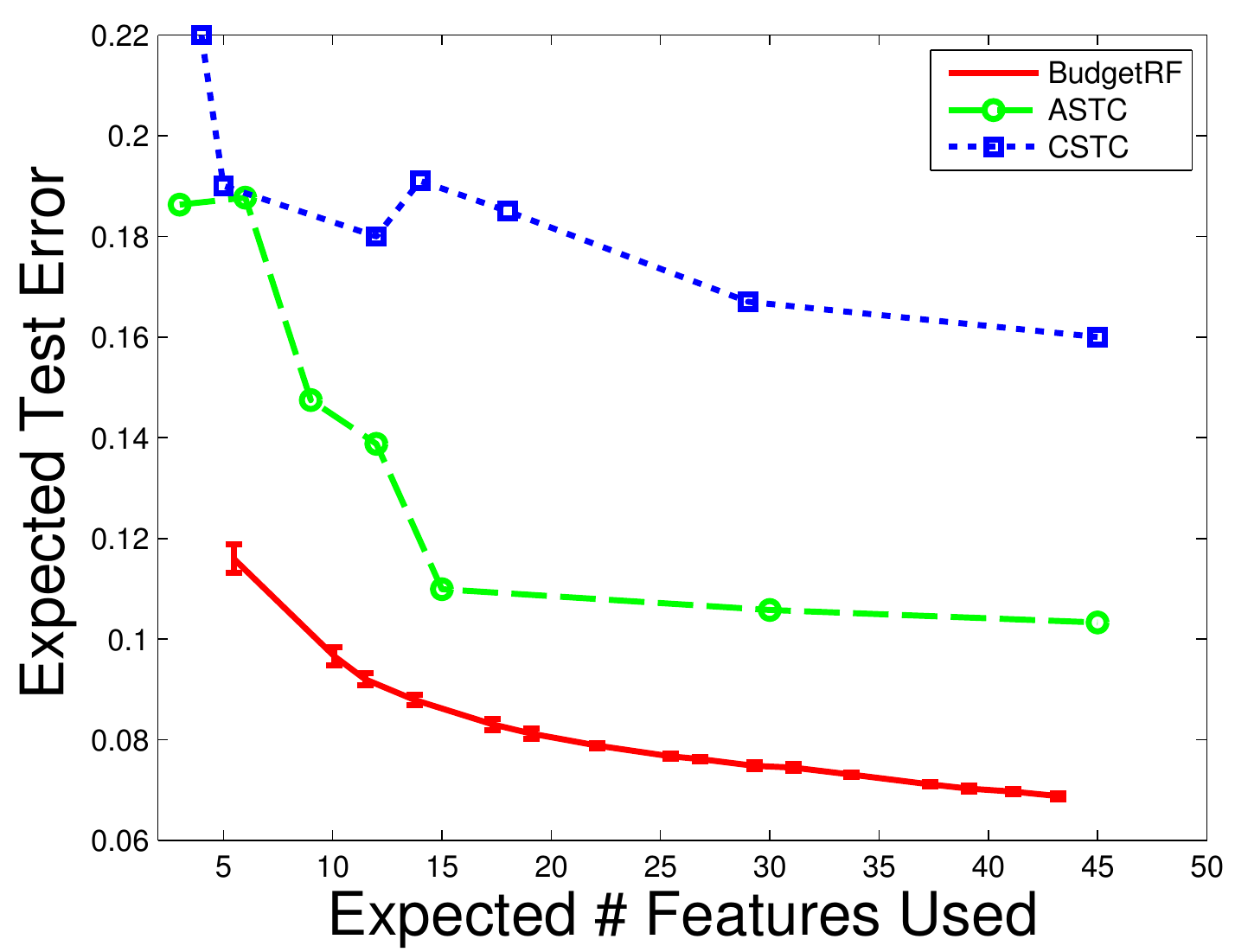}}\\
\vspace{-0.1in}
\subfigure[Forest Covertype]{\includegraphics[height=.29\linewidth]{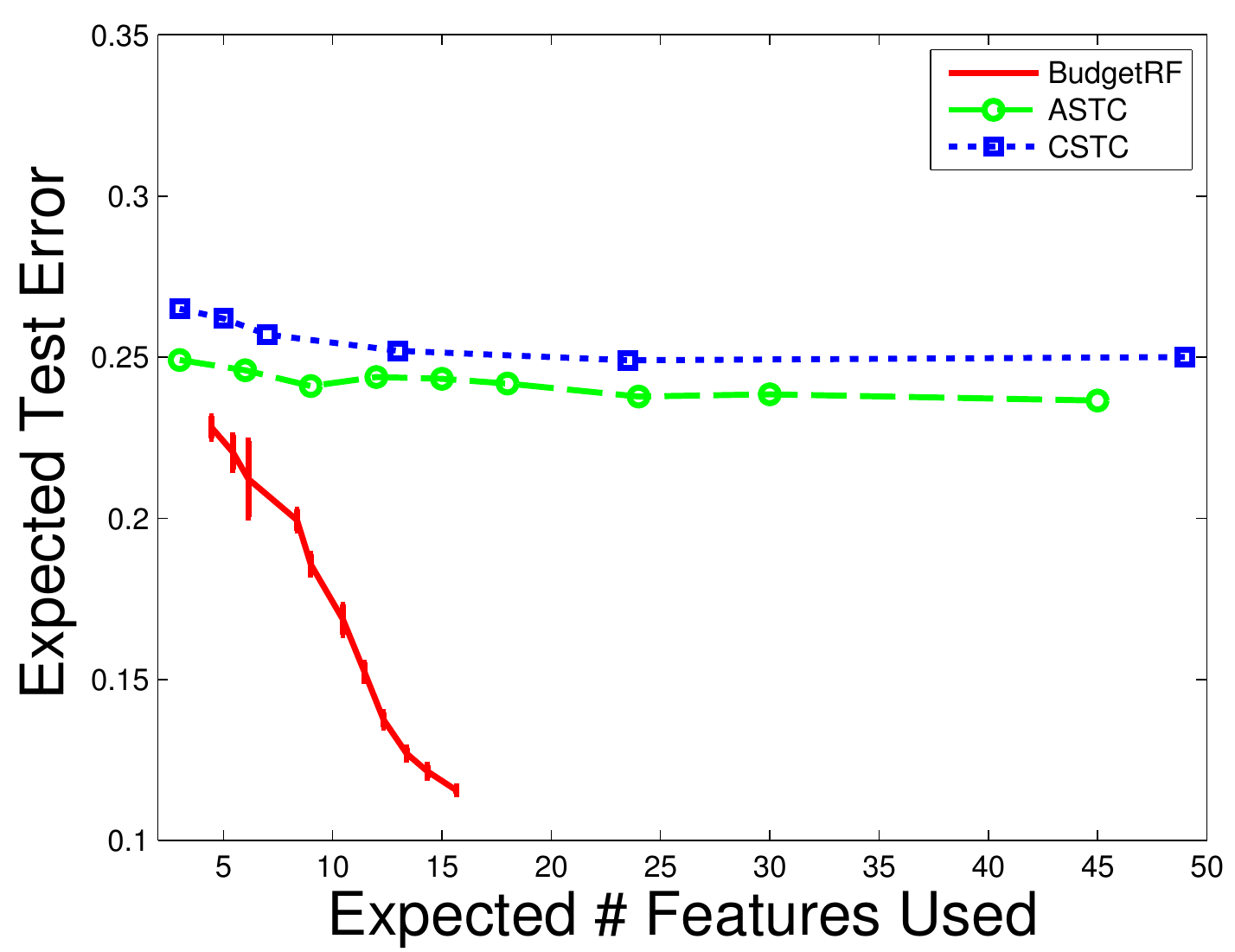}}
\subfigure[CIFAR-10]{\includegraphics[height=.29\linewidth]{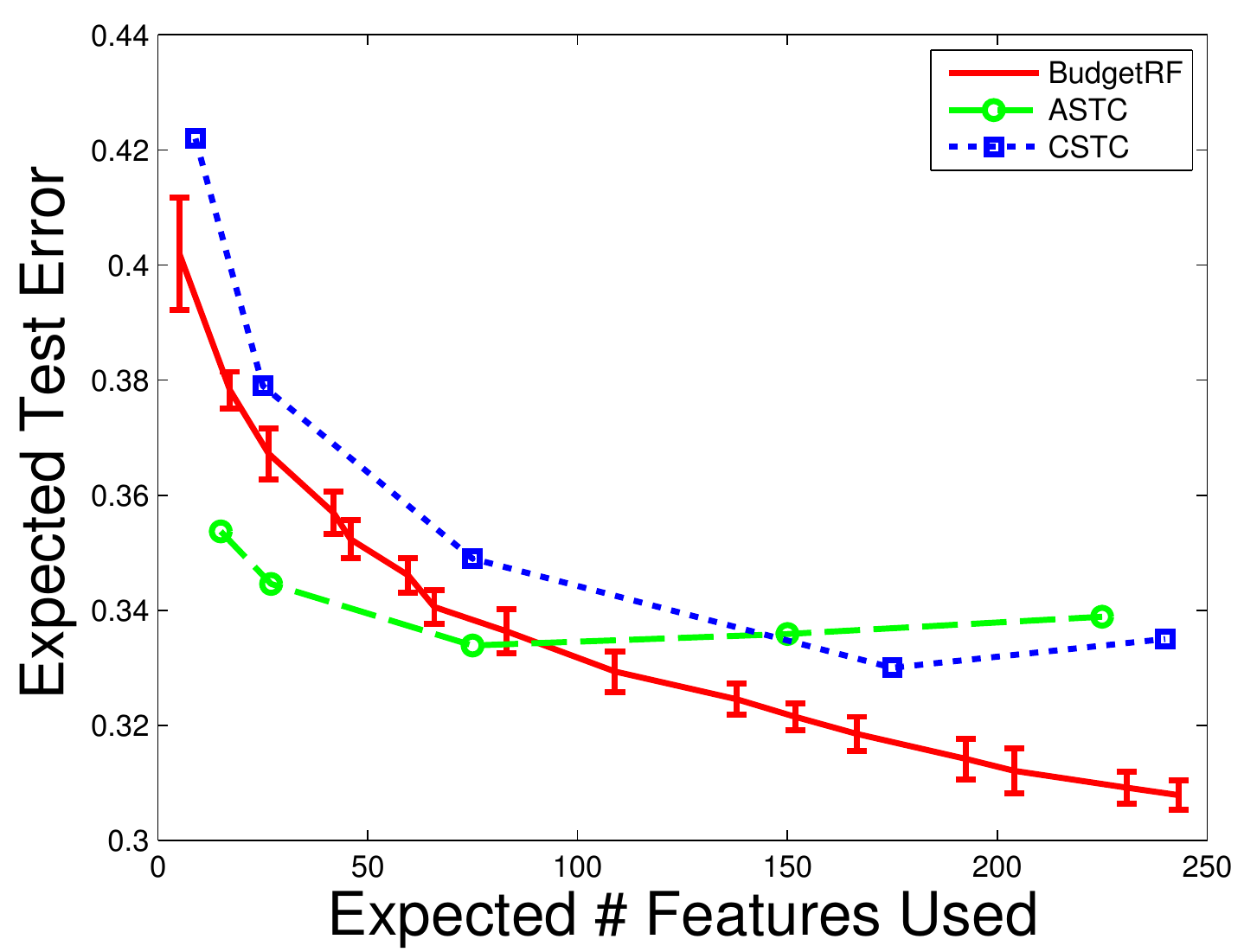}}
\vspace{-.35cm}
\caption{Comparison of \textsc{BudgetRF} against ASTC \cite{ASTC_AAAI14} and CSTC \cite{xu2013cost} on 4 real world datasets. \textsc{BudgetRF} has a clear advantage over these state-of-the-art methods as it achieves high accuracy/low error using less feature costs.}
\label{fig:experiments}
\vspace{-0.1in}
\end{figure*}

\textbf{Yahoo! Learning to Rank:} \cite{YahooChallenge2010} We evaluate \textsc{BudgetRF} on a real world budgeted learning problem: Yahoo! Learning to Rank Challenge \footnote{http://webscope.sandbox.yahoo.com/catalog.php?datatype=c}. The dataset consists of $473,134$ web documents and $19,944$ queries. Given a set of training query-document pairs together with relevance ranks of documents for each query, the Challenge is to learn an algorithm which takes a new query and its set of associated documents and outputs the rank of these documents with respect to the new query. Each example $x_i$ contains 519 features of a query-document pair. Each of these features is associated with an acquisition cost in the set $\{1, 5, 20, 50, 100, 150, 200\}$, which represents the units of time required for extraction and is provided by a Yahoo! employee. The labels are binarized so that $y_i=0$ means the document is unrelated to the query in $x_i$ whereas $y_i=1$ means the document is relevant to the query.  There are $141,397/146,769/184,968$ examples in training/validation/test sets.
We use the Average Precision@5 as performance metric, same as that used in \cite{ASTC_AAAI14}. To evaluate a predicted ranking for a test query, first sort the documents in decreasing order of the predicted ranks - that is, the more relevant documents predicted by the algorithm come before those that are deemed irrelevant. Take the top 5 documents in this order and reveal their true labels. If all of the documents are indeed relevant ($y=1$), then the precision score is increased by 1; otherwise, if the first unrelated document appears in position $1\leq j\leq 5$, increase the precision score by $\frac{j-1}{5}$. Finally, the precision score is averaged over the set of test queries. We run \textsc{BudgetRF} using the threshold $\alpha=0$ for the threshold-Pairs impurity function. To incorporate prediction confidence we simply run a given test example through the forest of trees to leaf nodes and aggregate the number of training examples at these leaf nodes for class $0$ and $1$ seperately. The ratio of class $1$ examples over the sum of class $1$ and $0$ examples gives the confidence of relevance. The comparison is shown in plot (a) of Figure \ref{fig:experiments}. The precision for \textsc{BudgetRF} rises much faster than ASTC and CSTC. At an average feature cost of 70, \textsc{BudgetRF} already exceeds the precision that ASTC/CSTC can achieve using feature cost of 450 and more. In this experiment the maximum number of trees we build is 140; the precision is set to rise even higher if we were to use more trees. \textsc{BudgetRF} thus represents a better ranking algorithm requiring much less wait time for users of the search engine.

\textbf{MiniBooNE Particle Identification Data Set:} \cite{UCI_repository}
The MiniBooNE data set is a binary classification task, with the goal of distinguishing electron neutrinos (signal) from muon neutrinos (background). Each data point consists of 50 experimental particle identification variables (features).
There are $45,523/19,510/65,031$ examples in training/validation/test sets.
We apply \textsc{BudgetRF} with a set of 10 values of $\alpha=[0,2,4,6,8,10,15,25,35,45]$. For each $\alpha$ we build a forest of maximum 40 trees using \textsc{BudgetRF}. Each point on the \textsc{BudgetRF} curve in (b) of Figure \ref{fig:experiments} corresponds to a $\alpha$ setting and the number of trees that meet the budget level.
The final $\alpha$ is chosen using validation set. Our algorithm clearly achieves lower test error than both ASTC and CSTC on every point of the budget level. Indeed, using just about 6 features on average out of 50 , \textsc{BudgetRF} achieves lower test error than what can be achieved by ASTC or CSTC using any number of features.

\textbf{Forest Covertype Data Set:} \cite{UCI_repository}
The Forest data set contains cartographic variables to predict 7 forest cover types. Each example contains 54 (10 continuous and 44 binary) features. There are $36,603/15,688/58,101$ examples in training/validation/test sets. 
We use the same $\alpha$ values as in MiniBooNE.
The final $\alpha$ is chosen using validation set.
In (c) of Figure \ref{fig:experiments}, ASTC and CSTC struggles to decrease test error even at high feature budget whereas the test error of \textsc{BudgetRF} decreases rapidly as more features are acquired. We believe this dramatic performance difference is partly due to the distinct advantage of \textsc{BudgetRF} in handling mixed continuous and discrete (categorical) data where the optimal decision function is highly non-linear.

\textbf{CIFAR-10:} \cite{CIFAR10}
CIFAR-10 data set consists of 32x32 colour images in 10 classes.
400 features for each image are extracted using technique described in \cite{ICML2011Coates_485}. 
The data are binarized by combining the first 5 classes into one class and the others into the second class. There are $19,761/8,468/10,000$ examples in training/validation/test sets. As shown in (d) of Figure \ref{fig:experiments} \textsc{BudgetRF} initially has higher test error than ASTC when the budget is low; from a budget about 90 onward \textsc{BudgetRF} outperforms ASTC while it outperforms CSTC on the entire curve. An important trend we see is that the errors for both ASTC and CSTC start to increase after some budget level. This indicates an issue of overfitting with these methods. We do not see such an issue with \textsc{BudgetRF}. 

As a general comment, we observe that in low-cost regions using higher $\alpha$ achieves lower test error whereas setting $\alpha=0$ leads to low test error at a higher cost. This is consistent with our intuition that setting a high value for $\alpha$ terminates the tree building process early and thus saves on cost, as a consequence more trees can be built within the budget. But as budget increases, more and more trees are added to the forest, the prediction power does not grow as fast as setting $\alpha$ to low values because the individual trees are not as powerful.

\textbf{Comments on standard Random Forest} Cost is not incorporated in the standard random forest (RF) algorithm. One issue that arises is how to incorporate budget constraint. Our strategy was to limit the number of trees in the RF to control the cost. But this does not work well even if the acquisition costs are uniform for all features. We implemented Matlab version of RF with the default settings on the Forest, MiniBooNE and CIFAR datasets: fraction of input data to sample with replacement from the input data for growing each new tree is 1; number of variables to select at random for each decision split is set to 8; minimum number of observations per tree leaf is 1. Compared to our \textsc{BudgetRF} algorithm using threshold-Pairs impurity with $\alpha=0$, the feature cost for RF is much higher as shown in Table \ref{table:RF}. For example in the Forest experiment, after building 10 trees, RF uses $63.04\%$ of total number of features for an average test example whereas \textsc{BudgetRF} uses only $23.21\%$. 
In terms of test error \textsc{BudgetRF} achieves $0.1364,0.0786$ and $0.3600$ for Forest, MiniBooNE and CIFAR respectively using 10 trees, quite competitive to $0.1318, 0.0803$ and $0.3594$ obtained by RF.\footnote{average over 10 repeated runs of RF and \textsc{BudgetRF}.} For Yahoo! Rank dataset, RF does even worse because some features have very high cost and yet RF still uses them just like the less expensive features, resulting in high cost.

\begin{center}
	\begin{table}
	\scalebox{0.53}{
	\begin{tabular}{| r | r | c | c | c |c | c | c |c | c | c |c|}
    \hline
     & Num. trees & 1 & 2 & 3 & 4 & 5 & 6 & 7 & 8 & 9 & 10 \\ \hline
     \multirow{2}{*}{Forest} & RF & 22.57  &  40.68  &  47.98  &  52.40  &  55.18  &  57.05  &  58.35  &  59.82   & 61.45  &  63.04  \\
     & BudgetRF & 11.37  &  15.55  &  17.96 &   19.40 &  20.47  &  21.21 &   21.85  &  22.37
     & 22.83 & 23.21 \\ \hline
     \multirow{2}{*}{MiniB} & RF & 26.86 &   42.92   & 54.59   & 63.74   & 70.24    & 75.15 &   78.84 &   81.45 &   83.27 &   85.73  \\
     & BudgetRF & 16.40 &   25.76 &  32.47  &  37.74 &  41.80 &   45.66 & 49.07 &  52.28    & 55.40 &   57.80 \\ \hline 
     \multirow{2}{*}{CIFAR} & RF & 3.86  &  7.49  &  10.98 &  14.24 &  17.38  &  20.34 &   22.84 &   25.43 &   28.02 &   30.51  \\
     & BudgetRF & 2.62 &  5.14  &  7.48  &  9.65  &  11.78  &  13.81  &  15.76 &  17.59 &   19.38  &  21.09 \\ \hline
    \end{tabular}
    }    
   	\caption[]{Percentage of average number of features used for different number of trees on Forest Covertype, MiniBooNE and CIFAR-10 datasets. BudgetRF uses much fewer features compared to RF.}
   	\label{table:RF}
    \end{table}
\end{center}

\section{Conclusion and Future Work}
We propose a novel algorithm to solve the budgeted learning problem. Our approach is to build a random forest of low cost trees with theoretical guarantees. We demonstrate that our algorithm performance far exceeds the state-of-the-art algorithms on 4 real world benchmarked datasets. While we have explored the greedy algorithm based on minimax-splits, similar algorithm can be proposed based on expected-splits. An interesting future work is to examine the theoretical and empirical properties of such algorithms.

\bibliography{cost_sensitive_bib}

\begin{thebibliography}{27}
\providecommand{\natexlab}[1]{#1}
\providecommand{\url}[1]{\texttt{#1}}
\expandafter\ifx\csname urlstyle\endcsname\relax
  \providecommand{\doi}[1]{doi: #1}\else
  \providecommand{\doi}{doi: \begingroup \urlstyle{rm}\Url}\fi

\bibitem[Bellala et~al.(2012)Bellala, Bhavnani, and
  Scott]{GroupBasedActiveLearning}
Bellala, G, Bhavnani, S.K., and Scott, C.
\newblock Group-based active query selection for rapid diagnosis in
  time-critical situations.
\newblock \emph{Information Theory, IEEE Transactions on}, Jan 2012.

\bibitem[Breiman(2001)]{Breiman}
Breiman, L.
\newblock Random forests.
\newblock \emph{Machine Learning}, 45\penalty0 (1):\penalty0 5--32, 2001.

\bibitem[Busa-Fekete et~al.(2012)Busa-Fekete, Benbouzid, and
  K{\'e}gl]{busa2012fast}
Busa-Fekete, R., Benbouzid, D., and K{\'e}gl, B.
\newblock Fast classification using sparse decision dags.
\newblock In \emph{Proceedings of the 29th International Conference on Machine
  Learning}, pp.\  951--958, 2012.

\bibitem[Chapelle et~al.()Chapelle, Chang, and Liu]{YahooChallenge2010}
Chapelle, O, Chang, Y, and Liu, T (eds.).
\newblock \emph{Proceedings of the Yahoo! Learning to Rank Challenge, held at
  {ICML} 2010, Haifa, Israel, June 25, 2010}.

\bibitem[Chen et~al.(2012)Chen, Xu, Weinberger, Chapelle, and Kedem]{chen:2012}
Chen, M., Xu, Z., Weinberger, K.~Q., Chapelle, O., and Kedem, D.
\newblock {Classifier cascade: Tradeoff between accuracy and feature evaluation
  cost}.
\newblock In \emph{International Conference on Artificial Intelligence and
  Statistics}, 2012.

\bibitem[Cicalese et~al.(2014)Cicalese, Laber, and
  Saettler]{DiagnosisDeterminationSimultaneous}
Cicalese, F, Laber, E.~S, and Saettler, A.~M.
\newblock Diagnosis determination: decision trees optimizing simultaneously
  worst and expected testing cost.
\newblock In \emph{Proceedings of the 31th International Conference on Machine
  Learning, {ICML} 2014, Beijing, China}, 2014.

\bibitem[Coates \& Ng(2011)Coates and Ng]{ICML2011Coates_485}
Coates, A. and Ng, A.~G.
\newblock The importance of encoding versus training with sparse coding and
  vector quantization.
\newblock In \emph{Proceedings of the 28th International Conference on Machine
  Learning}. ACM, 2011.

\bibitem[Cover \& Thomas(1991)Cover and Thomas]{cover}
Cover, T.~M. and Thomas, J.~A.
\newblock \emph{Elements of Information Theory}.
\newblock Wiley-Interscience, New York, NY, USA, 1991.

\bibitem[Dulac-Arnold et~al.(2011)Dulac-Arnold, Denoyer, Preux, and
  Gallinari]{dulac2011datum}
Dulac-Arnold, G., Denoyer, L., Preux, P., and Gallinari, P.
\newblock Datum-wise classification: a sequential approach to sparsity.
\newblock In \emph{Machine Learning and Knowledge Discovery in Databases}, pp.\
   375--390. 2011.

\bibitem[Frank \& Asuncion()Frank and Asuncion]{UCI_repository}
Frank, A. and Asuncion, A.
\newblock {UCI} machine learning repository.

\bibitem[Gao \& Koller(2011)Gao and Koller]{Gao+Koller:NIPS11}
Gao, T. and Koller, D.
\newblock Active classification based on value of classifier.
\newblock In \emph{Advances in Neural Information Processing Systems (NIPS
  2011)}, 2011.

\bibitem[He et~al.(2012)He, Daume~III, and Eisner]{he2012imitation}
He, H, Daume~III, H, and Eisner, J.
\newblock Imitation learning by coaching.
\newblock In \emph{Advances In Neural Information Processing Systems}, 2012.

\bibitem[Ji \& Carin(2007)Ji and Carin]{ji:2007}
Ji, S and Carin, L.
\newblock {Cost-sensitive feature acquisition and classification}.
\newblock \emph{Pattern Recognition}, 2007.

\bibitem[Kanani \& Melville(2008)Kanani and Melville]{kanani_melville}
Kanani, P. and Melville, P.
\newblock {Prediction-time Active Feature-Value Acquisition for Cost-Effective
  Customer Targeting}.
\newblock In \emph{Advances In Neural Information Processing Systems (NIPS)},
  2008.

\bibitem[Kapoor \& Horvitz(2009)Kapoor and Horvitz]{kapoor:2009}
Kapoor, A and Horvitz, E.
\newblock Breaking boundaries: Active information acquisition across learning
  and diagnosis.
\newblock In \emph{Advances In Neural Information Processing Systems}, 2009.

\bibitem[Karayev et~al.(2013)Karayev, Fritz, and Darrell]{karayev13}
Karayev, S, Fritz, M, and Darrell, T.
\newblock Dynamic feature selection for classification on a budget.
\newblock In \emph{International Conference on Machine Learning (ICML):
  Workshop on Prediction with Sequential Models}, 2013.

\bibitem[Krizhevsky(2009)]{CIFAR10}
Krizhevsky, Alex.
\newblock {Learning Multiple Layers of Features from Tiny Images}.
\newblock Master's thesis, 2009.
\newblock URL
  \url{http://www.cs.toronto.edu/\~{}kriz/learning-features-2009-TR.pdf}.

\bibitem[Kusner et~al.(2014)Kusner, Chen, Zhou, Zhixiang, Weinberger, and
  Chen]{ASTC_AAAI14}
Kusner, M, Chen, W, Zhou, Q, Zhixiang, E, Weinberger, K, and Chen, Y.
\newblock Feature-cost sensitive learning with submodular trees of classifiers.
\newblock In \emph{{AAAI} Conference on Artificial Intelligence}, 2014.

\bibitem[MacKay(1992)]{mackay}
MacKay, D. J.~C.
\newblock {Information-based objective functions for active data selection}.
\newblock \emph{Neural computation}, 4\penalty0 (4):\penalty0 590--604, 1992.

\bibitem[Moshkov(2010)]{MoshkovGreedyAlgorithmwithWeightsforDecisionTreeConstruction}
Moshkov, M.~J.
\newblock Greedy algorithm with weights for decision tree construction.
\newblock \emph{Fundam. Inf.}, 104\penalty0 (3):\penalty0 285--292, August
  2010.

\bibitem[Nowak(2008)]{Nowak08generalizedbinary}
Nowak, Robert.
\newblock Generalized binary search.
\newblock In \emph{In Proceedings of the 46th Allerton Conference on
  Communications, Control, and Computing}, pp.\  568--574, 2008.

\bibitem[Trapeznikov \& Saligrama(2013)Trapeznikov and
  Saligrama]{trapeznikov:2013b}
Trapeznikov, K and Saligrama, V.
\newblock Supervised sequential classification under budget constraints.
\newblock In \emph{International Conference on Artificial Intelligence and
  Statistics}, pp.\  581--589, 2013.

\bibitem[Viola \& Jones(2001)Viola and Jones]{viola01}
Viola, P and Jones, M.
\newblock {Robust Real-time Object Detection}.
\newblock \emph{International Journal of Computer Vision}, 4:\penalty0 34--47,
  2001.

\bibitem[Wang et~al.(2014{\natexlab{a}})Wang, Bolukbasi, Trapeznikov, and
  Saligrama]{wang2014model}
Wang, J., Bolukbasi, T., Trapeznikov, K, and Saligrama, V.
\newblock Model selection by linear programming.
\newblock In \emph{European Conference on Computer Vision}, pp.\  647--662,
  2014{\natexlab{a}}.

\bibitem[Wang et~al.(2014{\natexlab{b}})Wang, Trapeznikov, and
  Saligrama]{wang2014lp}
Wang, J, Trapeznikov, K, and Saligrama, V.
\newblock An lp for sequential learning under budgets.
\newblock In \emph{International Conference on Artificial Intelligence and
  Statistics}, 2014{\natexlab{b}}.

\bibitem[Xu et~al.(2013)Xu, Kusner, Chen, and Weinberger]{xu2013cost}
Xu, Z, Kusner, M, Chen, M, and Weinberger, K.~Q.
\newblock Cost-sensitive tree of classifiers.
\newblock In \emph{Proceedings of the 30th International Conference on Machine
  Learning}, 2013.

\bibitem[{}Zhang \& {}Zhang(2010){}Zhang and {}Zhang]{zhang:2010}
{}Zhang, C.\ and {}Zhang, Z.\.
\newblock {A Survey of Recent Advances in Face Detection}.
\newblock Technical report, Microsoft Research, 2010.

\end{thebibliography}
\bibliographystyle{icml2015}
\onecolumn
\section*{Appendix}
\paragraph{Proof of Lemma 2.3} 
Before showing admissibility of the threshold-Pairs function in the multiclass setting, we first show  $F_\alpha(G)$ is admissible for the binary setting.
Consider the binary classification setting, let
\begin{equation*}
F_\alpha(G)=[[n^1_G-\alpha]_+[n^2_G-\alpha]_+-\alpha^2]_+.
\end{equation*}
All the properties are obviously true except supermodularity. To show supermodularity, suppose $R\subseteq G$ and object $j\notin R$. Suppose $j$ belongs to the first class. We need to show
\begin{equation}
F_\alpha(G\cup j)-F_\alpha(G)\geq F_\alpha(R\cup j) -F_\alpha(R). \label{eq:lemma5_1}
\end{equation}
Consider 3 cases:\\
(1) $F_\alpha(R)=F_\alpha(R\cup j)=0$: The right hand side of \eqref{eq:lemma5_1} is 0 and \eqref{eq:lemma5_1} holds because of monotonicity of $F_\alpha$.\\
(2) $F_\alpha(R)=0,F_\alpha(R\cup j)>0,F_\alpha(G)=0$: \eqref{eq:lemma5_1} reduces to $F_\alpha(G\cup j)\geq F_\alpha(R\cup j)$, which is true by monotonicity. \\
(3) $F_\alpha(R)=0,F_\alpha(R\cup j)>0,F_\alpha(G)>0$: Note that $F_\alpha(G)>0$ implies that $[n^1_G-\alpha]_+[n^2_G-\alpha]_+-\alpha^2>0$ which further implies $n_G^1>\alpha, n_G^2>\alpha$. Thus the left hand side is
\begin{align*}
& F_\alpha(G\cup j)-F_\alpha(G)= \\
&(n_G^1-\alpha+1)(n_G^2-\alpha)-\alpha^2-((n_G^1-\alpha)(n_G^2-\alpha)-\alpha^2)\\
&=n_G^2-\alpha.
\end{align*}
The right hand side is
\begin{align*}
 F_\alpha(R\cup j) &=(n^1_R-\alpha+1)(n^2_R-\alpha)-\alpha^2 \\
&=(n^1_R-\alpha)(n^2_R-\alpha)-\alpha^2+(n^2_R-\alpha).
\end{align*}
If $n_R^1\geq \alpha$, $ F_\alpha(R)=\max((n^1_R-\alpha)(n^2_R-\alpha)-\alpha^2,0)=0$ because $F_\alpha(R\cup j)>0$ implies $n_R^2>\alpha$. So $F_\alpha(R\cup j)\leq n^2_R-\alpha \leq n_G^2-\alpha = F_\alpha(G\cup j)-F_\alpha(G)$.\\
(4) $F_\alpha(R)>0$: We have
\begin{equation*}
F_\alpha(G\cup j)-F_\alpha(G)  = n_G^2-\alpha \geq n_R^2 - \alpha  = F_\alpha(R\cup j) -F_\alpha(R).
\end{equation*}
This completes the proof for the binary classification setting.
To generalize to the multiclass threshold-Pairs function, again, all properties are obviously true except supermodularity, which follows from the fact that each term in the sum is supermodular according to the proof for binary setting.

\paragraph{More Admissible Impurity Functions} The following polynomial impurity function is also admissible. 
\begin{lemma} \label{lemma:polynomialAdmissible}
Suppose there are $k$ classes in $G$. Any polynomial function of $n_G^1,\dots,n_G^k$ with non-negative terms such that $n_G^1,\dots,n_G^k$ do not appear as singleton terms is \emph{admissible}. Formally, if
\begin{equation}\label{eq:poly}
F(G)=\sum_{i=1}^{M} \gamma_i (n_G^1)^{p_{i1}}(n_G^2)^{p_{i2}}\dots (n_G^k)^{p_{ik}},
\end{equation}
where $\gamma_i$'s are non-negative, $p_{ij}$'s are non-negative integers and for each $i$ there exists at least 2 non-zero $p_{ij}$'s, then $F$ is \emph{admissible}.
\end{lemma}
\begin{proof}
Properties (1),(2),(3) and (5) are obviously true. To show $F$ is supermodular, suppose $R\subset G$ and object $\hat{j} \notin R$ and $\hat{j}$ belongs to class $j$, we have
\begin{align*}
& F(R\cup \hat{j})-F(R)\\
& =\sum_{i\in I_j} \gamma_i [(n_R^1)^{p_{i1}}\dots(n_R^j+1)^{p_{ij}}\dots (n_R^k)^{p_{ik}}-\\
& \qquad (n_R^1)^{p_{i1}}\dots(n_R^j)^{p_{ij}}\dots (n_R^k)^{p_{ik}}]\\
& \leq \sum_{i\in I_j} \gamma_i [(n_G^1)^{p_{i1}}\dots(n_G^j+1)^{p_{ij}}\dots (n_G^k)^{p_{ik}}-\\
& \qquad (n_G^1)^{p_{i1}}\dots(n_G^j)^{p_{ij}}\dots (n_G^k)^{p_{ik}}]\\
&=F(G\cup \hat{j})-F(G),
\end{align*}
where the first summation index set $I_j$ is the set of terms that involve $n_R^j$. The inequality follows because $(n_R^j+1)^{p_{ij}}$ can be expanded so the negative term can be canceled, leaving a sum-of-products form for $R$, which is term-by-term dominated by that of $G$.
\end{proof}

Another family of \emph{admissible} impurity functions is the Powers function.
\begin{corollary}\label{cor:powerW}
Powers function
\begin{equation}\label{eq:powerFunc}
F(G)=(\sum_{i=1}^{k} n_G^i)^l - \sum_{i=1}^{k}(n_G^i)^l
\end{equation}
is \emph{admissible} for $l=2,3,\dots$.
\end{corollary}

We compare the threshold-Pairs with various $\alpha$ values against the Powers function to study the effect of them on the tree building subroutine \textsc{GreedyTree}. We compare performance using 9 data sets from the UCI Repository in Figure \ref{fig:real_world_tradeoff_curves}. 
We assume that all features have a uniform cost. For each data set, we replace non-unique objects with a single instance using the most common label for the objects, allowing every data set to be complete (perfectly classified by the decision trees). Additionally, continuous features are transformed to discrete features by quantizing to 10 uniformly spaced levels.
For trees with a smaller cost (and therefore lower depth), the threshold-Pairs impurity function outperforms the Powers impurity function with early stopping (higher $\alpha$ leads to earlier stopping), whereas for larger cost (and greater depth), the Powers impurity function outperforms threshold-Pairs. If $\alpha$ is set to 0, the difference between threshold-Pairs and Powers function is small.
\paragraph{Details of Data Sets}
 The house votes data set is composed of the voting records for 435 members of the U.S. House of Representatives (342 unique voting records) on 16 measures, with a goal of identifying the party of each member. The sonar data set contains 208 sonar signatures, each composed of energy levels (quantized to 10 levels) in 60 different frequency bands, with a goal of identifying  The ionosphere data set has 351 (350 unique) radar returns, each composed of 34 responses (quantized to 10 levels), with a goal of identifying if an event represents a free electron in the ionosphere. The Statlog DNA data set is composed of 3186 (3001 unique) DNA sequences with 180 features, with a goal of predicting whether the sequence represents a boundary of DNA to be spliced in or out. The Boston housing data set contains 13 attributes (quantized to 10 levels) pertaining to 506 (469 unique) different neighborhoods around Boston, with a goal of predicting which quartile the median income of the neighborhood the neighborhood falls. The soybean data set is composed of 307 examples (303 unique) composed of 34 categorical features, with a goal of predicting from among 19 diseases which is afflicting the soy bean plant. The pima data set is composed of 8 features (with continuous features quantized to 10 levels) corresponding to medical information and tests for 768 patients (753 unique feature patterns), with a goal of diagnosing diabetes. The Wisconsin breast cancer data set contains 30 features corresponding to properties of a cell nucleus for 569 samples, with a goal of identifying if the cell is malignant or benign. The mammography data set contains 6 features from mammography scans (with age quantized into 10 bins) for 830 patients, with a goal of classifying the lesions as malignant or benign.
\begin{figure*}[htb!]
\centering
\subfigure[House Votes]{\includegraphics[trim=40mm 88mm 40mm 90mm,clip,height=.23\linewidth]{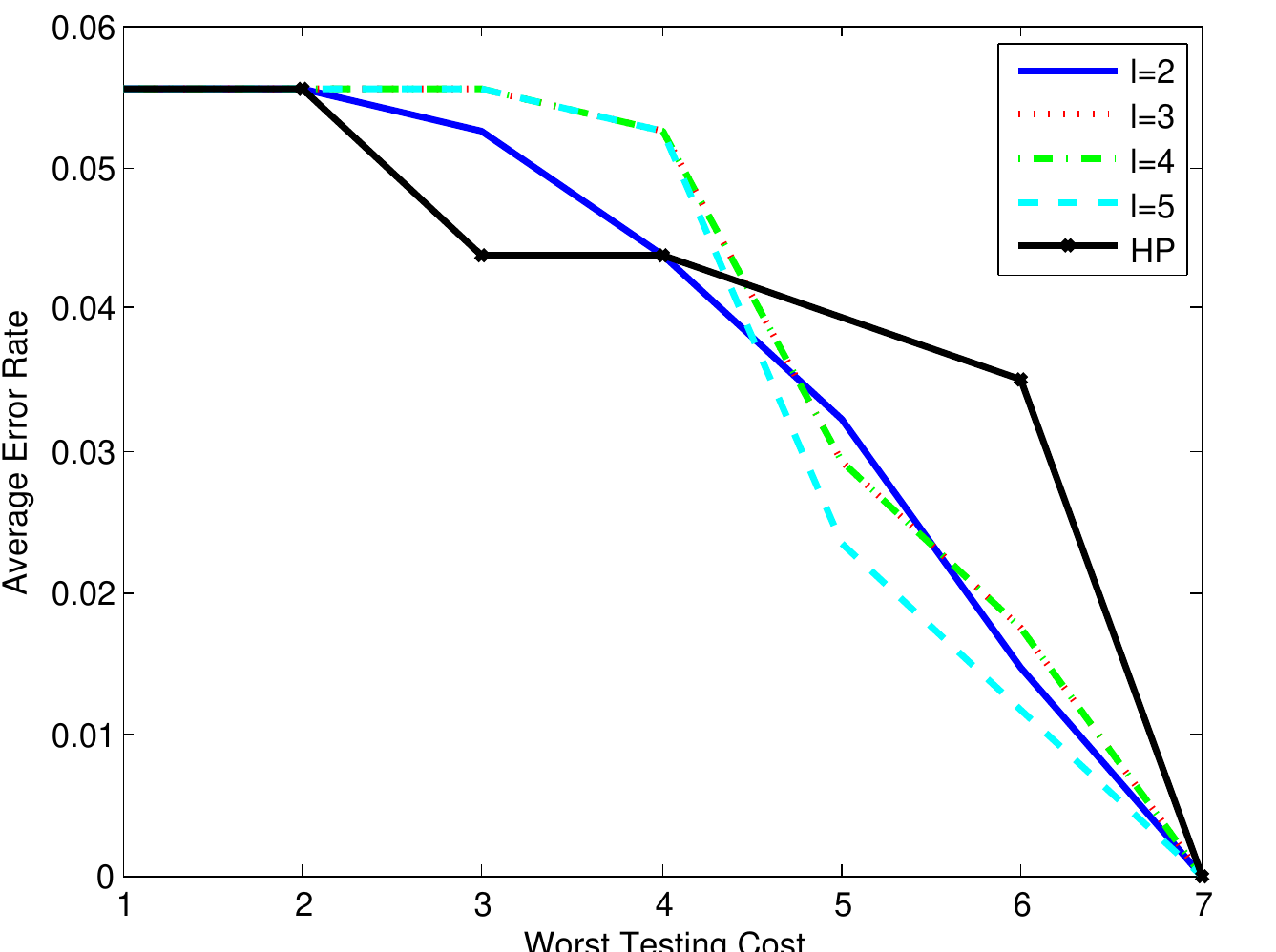}}
\subfigure[Sonar]{\includegraphics[trim=40mm 88mm 40mm 90mm,clip,height=.23\linewidth]{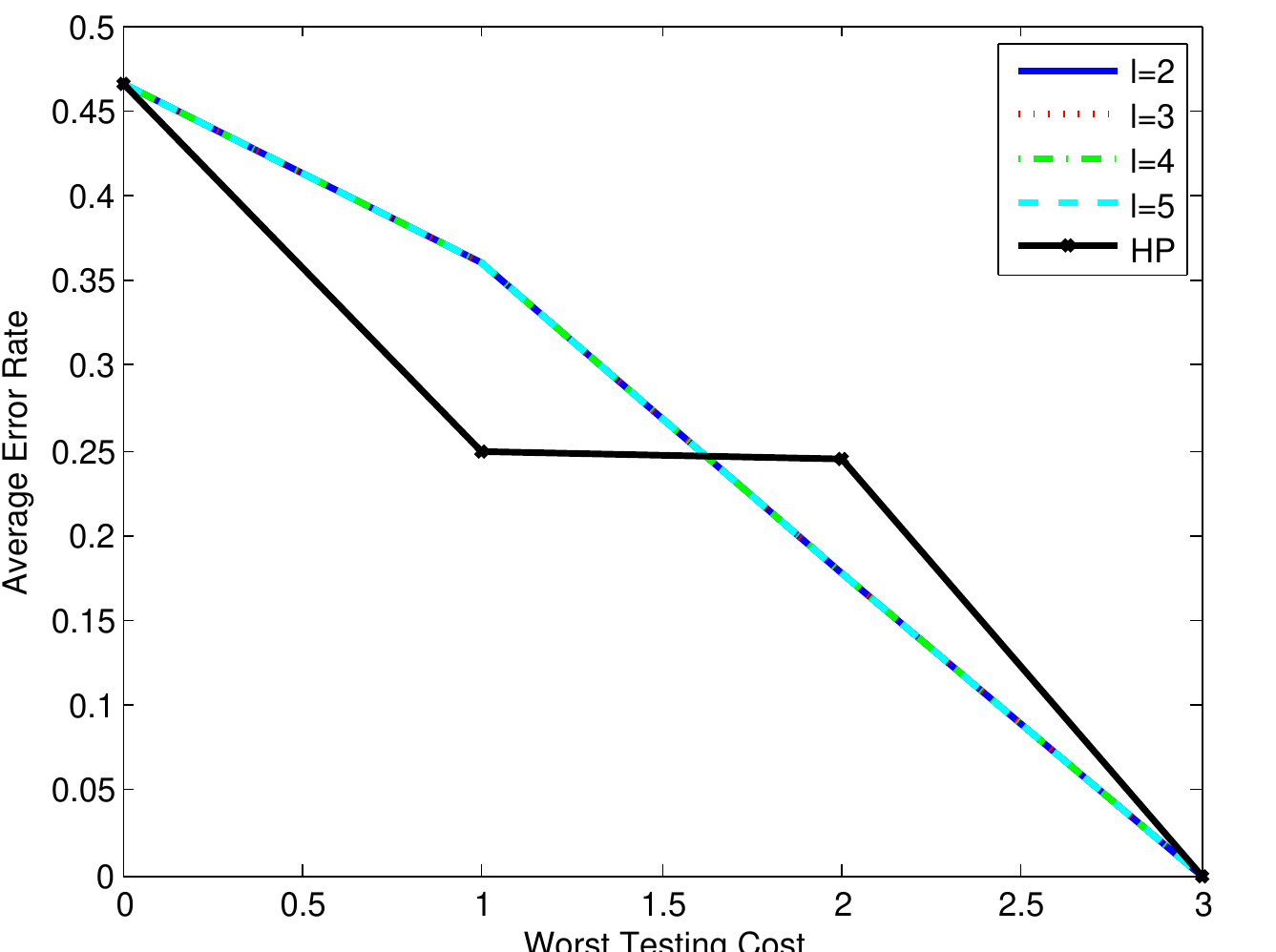}}
\subfigure[Ionosphere]{\includegraphics[trim=40mm 88mm 40mm 90mm,clip,height=.23\linewidth]{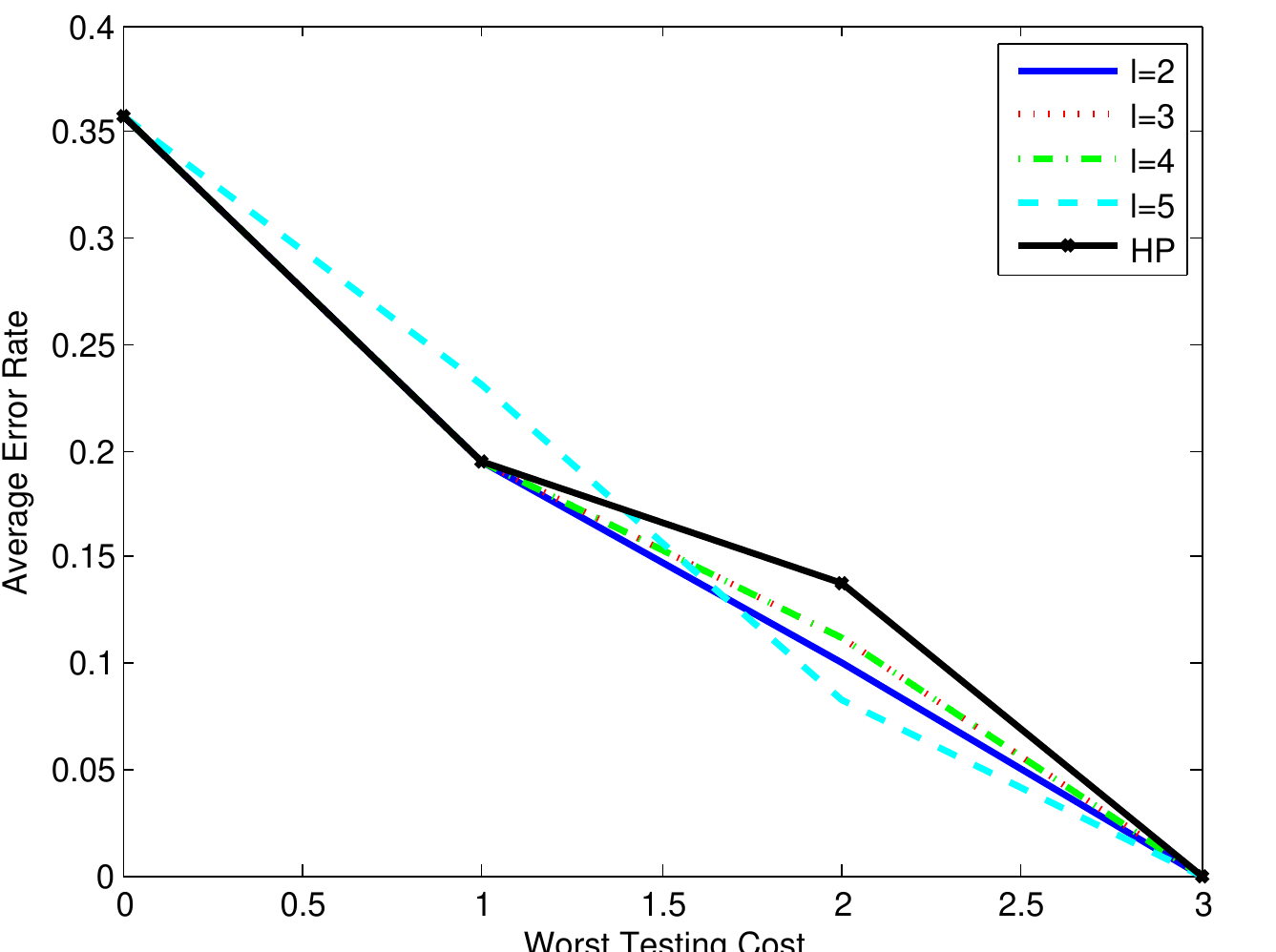}}\\[-2ex]
\subfigure[Statlog DNA]{\includegraphics[trim=40mm 88mm 40mm 90mm,clip,height=.23\linewidth]{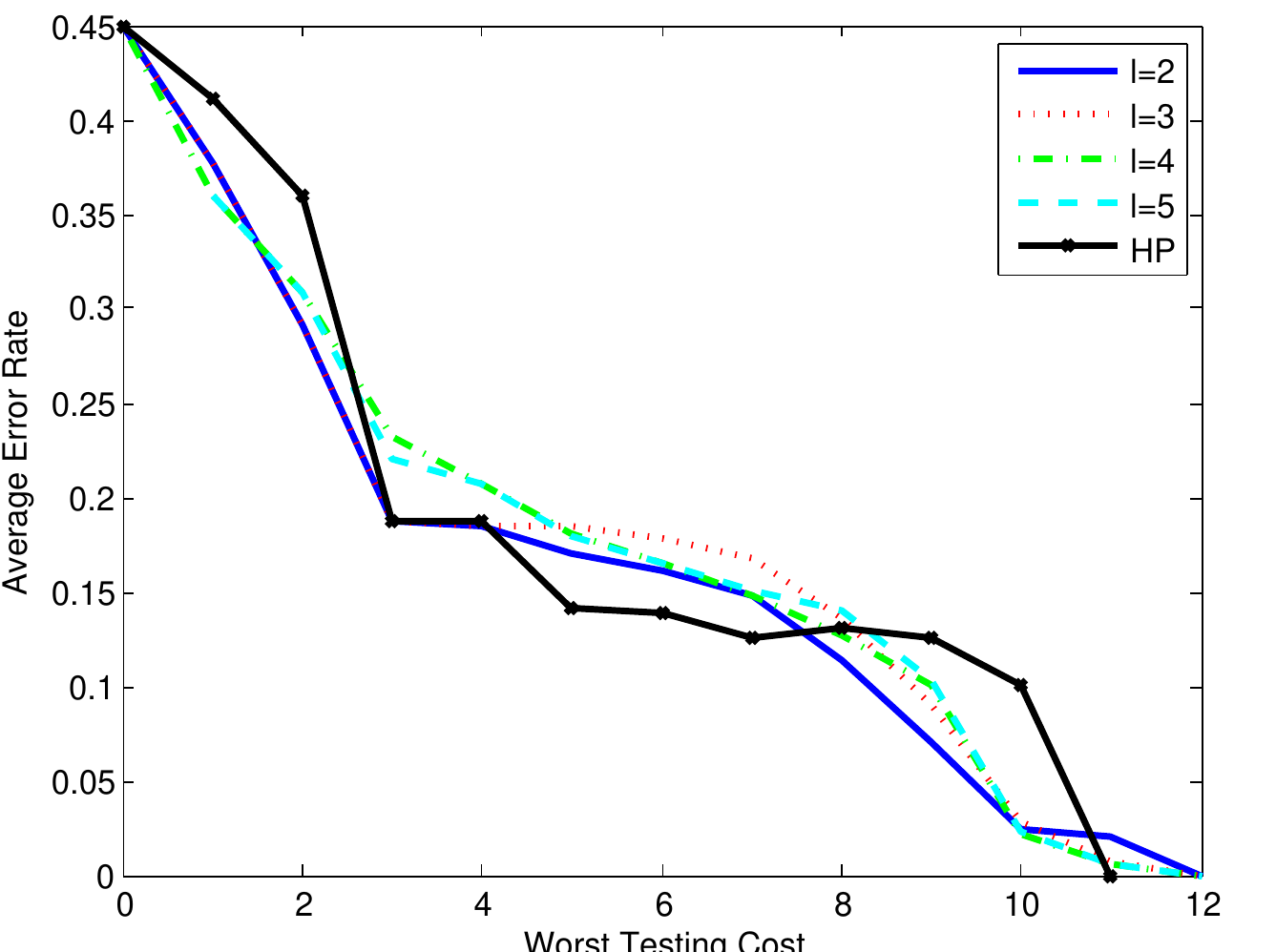}}
\subfigure[Boston Housing]{\includegraphics[trim=40mm 88mm 40mm 90mm,clip,height=.23\linewidth]{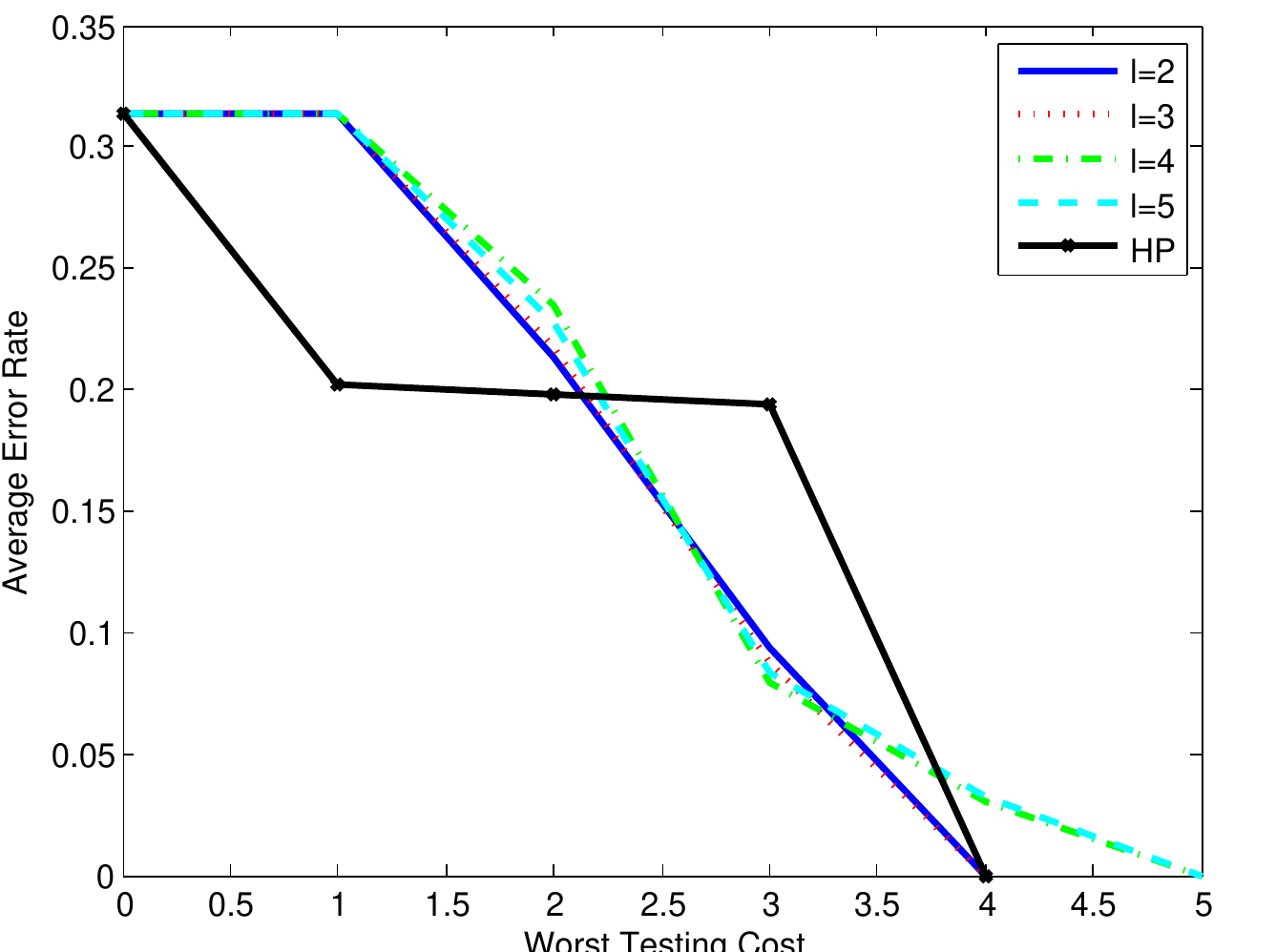}}
\subfigure[Soybean]{\includegraphics[trim=40mm 88mm 40mm 90mm,clip,height=.23\linewidth]{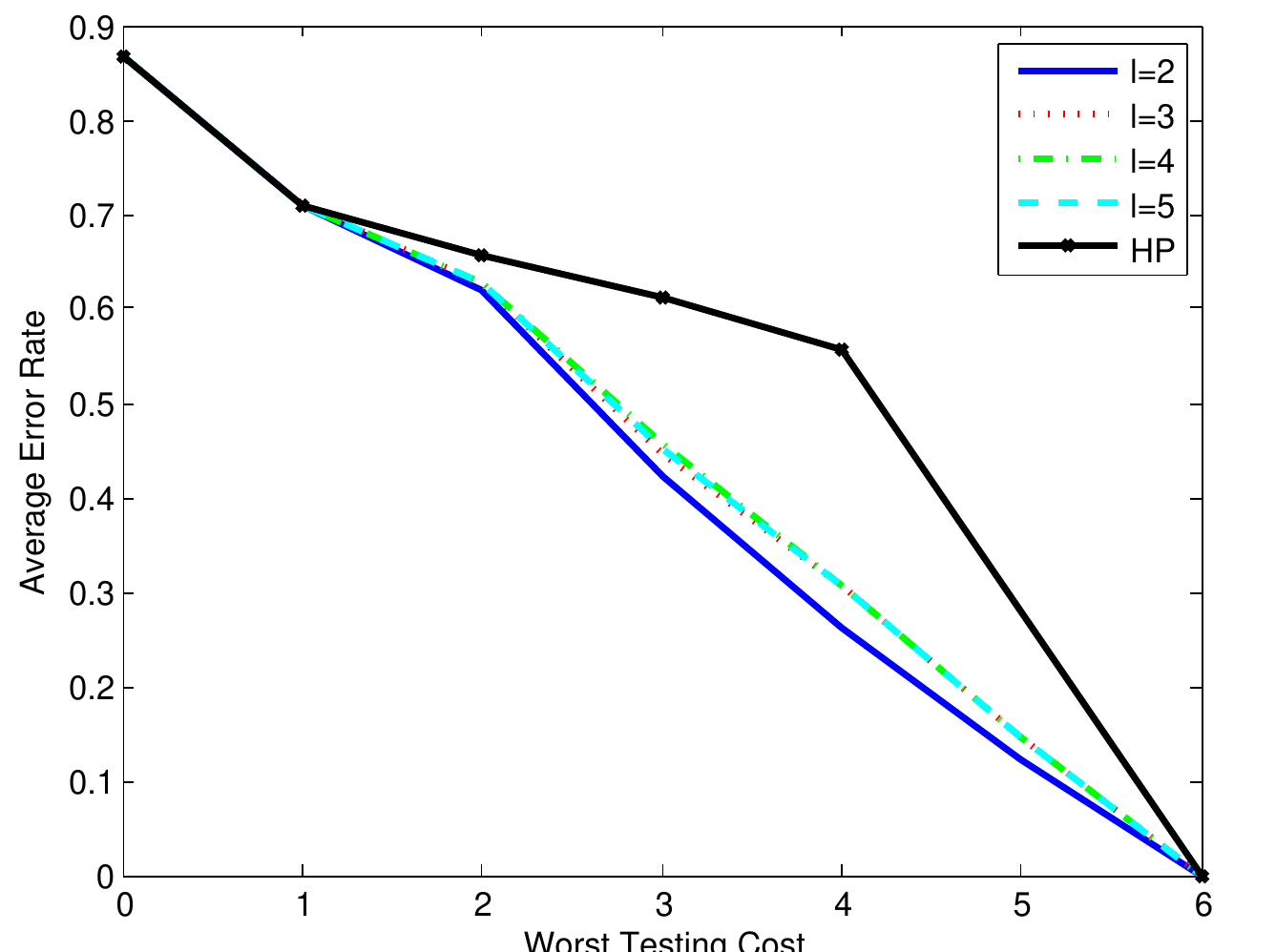}}\\[-2ex]
\subfigure[Pima]{\includegraphics[trim=40mm 88mm 40mm 90mm,clip,height=.23\linewidth]{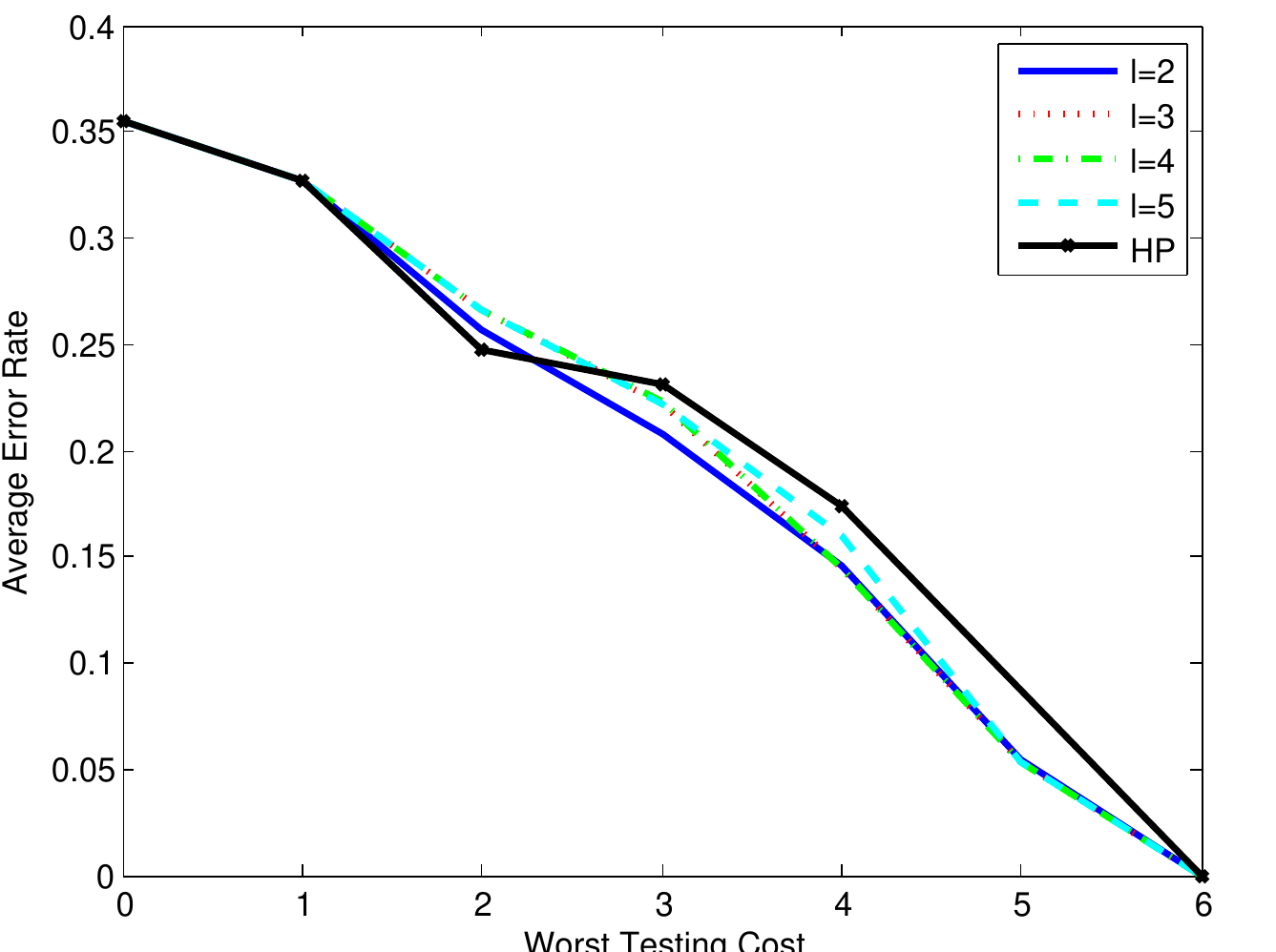}}
\subfigure[WBCD]{\includegraphics[trim=40mm 88mm 40mm 90mm,clip,height=.23\linewidth]{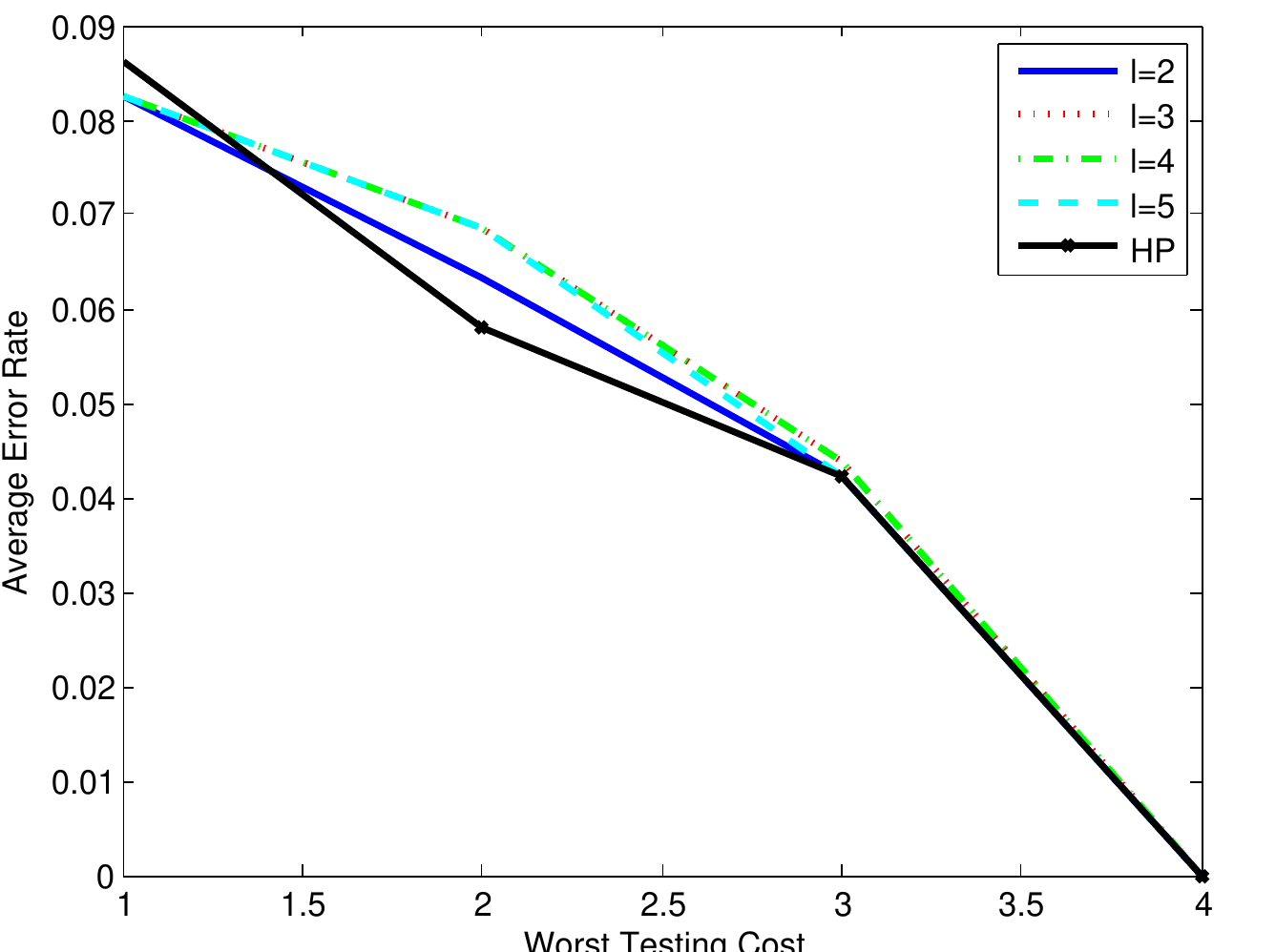}}
\subfigure[Mammography]{\includegraphics[trim=40mm 88mm 40mm 90mm,clip,height=.23\linewidth]{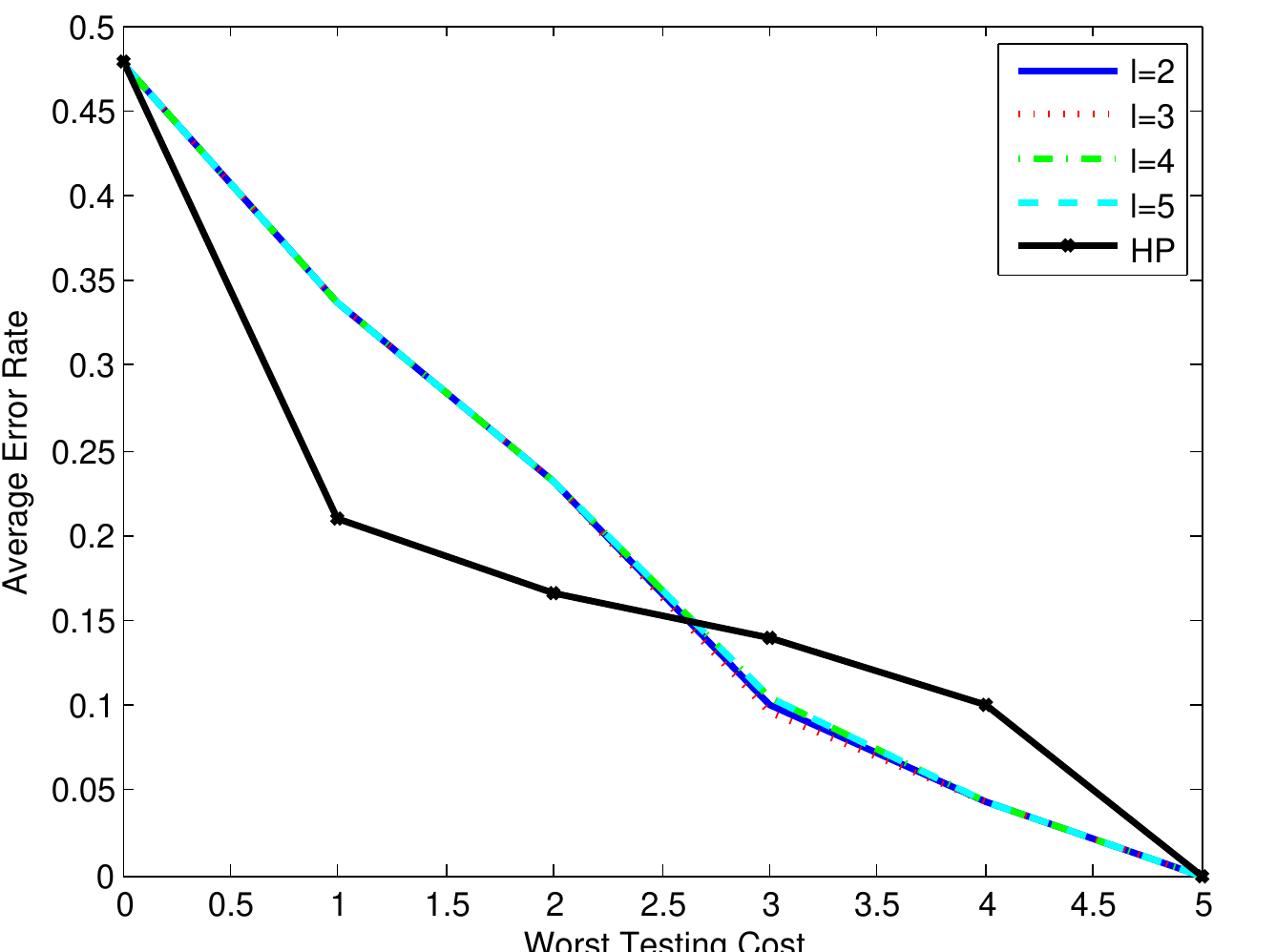}}
\vspace{-.35cm}
\caption{Comparison of classification error vs. max-cost for the Powers impurity function in \eqref{eq:powerFunc} for $l=2,3,4,5$ and the threshold-Pairs impurity function. Note that for both House Votes and WBCD, the depth $0$ tree is not included as the error decreases dramatically using a single test. In many cases, the threshold-Pairs impurity function outperforms the Powers impurity functions for trees with smaller max-costs, whereas the Powers impurity function outperforms the threshold-Pairs function for larger max-costs.}
\label{fig:real_world_tradeoff_curves}
\end{figure*}

\paragraph{Details of Computation in Figure 1}
If $\alpha=0$, we can compute impurity of each set of interest: $F_0(G)=30\times 30=900, F_0(G_{t_1}^1)=30\times 10=300, F_0(G_{t_1}^2)=0, F_0(G_{t_2}^1)=F_0(G_{t_2}^2)=15\times 15=225$; according to subroutine \textsc{GreedyTree}, we can compute $R(t_1)=\max \{\frac{1}{900-300},\frac{1}{900-0}\}=\frac{1}{600}, R(t_2)=\max \{\frac{1}{900-225},\frac{1}{900-225}=\frac{1}{675}\}$ so $t_2$ will be chosen. On the other hand, the impurities for the threshold-Pairs with $\alpha=8$ are $F_8(G)=22\times 22=484, F_8(G_{t_1}^1)=22\times 2=44, F_8(G_{t_1}^2)=0, F_8(G_{t_2}^1)=F_8(G_{t_2}^2)=7\times 7=49$; again we can compute $R(t_1)=\max \{\frac{1}{484-44},\frac{1}{484-0}\}=\frac{1}{440}, R(t_2)=\max \{\frac{1}{484-49},\frac{1}{484-49}=\frac{1}{435}\}$ so $t_1$ will be chosen. The above example shows that setting $\alpha=0$ has a stronger preference to balanced splits and may in some cases lead to poor classification result.

\end{document}